\documentclass[sigconf]{acmart}

\usepackage{multirow}
\usepackage{subfigure}
\usepackage{algorithm}
\usepackage{algpseudocode}  

\AtBeginDocument{%
  \providecommand\BibTeX{{%
    \normalfont B\kern-0.5em{\scshape i\kern-0.25em b}\kern-0.8em\TeX}}}

\copyrightyear{2020} 
\acmYear{2020} 
\setcopyright{acmcopyright}\acmConference[CIKM '20]{Proceedings of the 29th ACM International Conference on Information and Knowledge Management}{October 19--23, 2020}{Virtual Event, Ireland}
\acmBooktitle{Proceedings of the 29th ACM International Conference on Information and Knowledge Management (CIKM '20), October 19--23, 2020, Virtual Event, Ireland}
\acmPrice{15.00}
\acmDOI{10.1145/3340531.3411963}
\acmISBN{978-1-4503-6859-9/20/10}

\begin{document}

\fancyhead{}

\title{Streaming Graph Neural Networks via Continual Learning}

\author{Junshan Wang}
\affiliation{%
  \institution{Key Laboratory of Machine Perception, Ministry of Education, Peking University}
  \city{Beijing}
  \state{China}
}
\email{wangjunshan@pku.edu.cn}

\author{Guojie Song}
\authornote{Corresponding author.}
\affiliation{%
  \institution{Key Laboratory of Machine Perception, Ministry of Education, Peking University}
  \city{Beijing}
  \state{China}
}
\email{gjsong@pku.edu.cn}

\author{Yi Wu}
\affiliation{%
  \institution{Peking University}
  \city{Beijing}
  \state{China}
}
\email{reyn19990619@pku.edu.cn}

\author{Liang Wang}
\affiliation{%
 \institution{Alibaba group}
 \city{Beijing}
 \state{China}
}
\email{liangbo.wl@alibaba-inc.com}


\begin{abstract}
    \textit{Graph neural networks (GNNs)} have achieved strong performance in various applications. 
    In the real world, network data is usually formed in a streaming fashion. The distributions of patterns that refer to neighborhood information of nodes may shift over time. The GNN model needs to learn the new patterns that cannot yet be captured. But learning incrementally leads to the catastrophic forgetting problem that historical knowledge is overwritten by newly learned knowledge. Therefore, it is important to train GNN model to learn new patterns and maintain existing patterns simultaneously, which few works focus on. 
    In this paper, we propose a streaming GNN model based on continual learning so that the model is trained incrementally and up-to-date node representations can be obtained at each time step. 
    Firstly, we design an approximation algorithm to detect new coming patterns efficiently based on information propagation. Secondly, we combine two perspectives of data replaying and model regularization for existing pattern consolidation. Specially, a hierarchy-importance sampling strategy for nodes is designed and a weighted regularization term for GNN parameters is derived, achieving greater stability and generalization of knowledge consolidation. 
    Our model is evaluated on real and synthetic data sets and compared with multiple baselines. The results of node classification prove that our model can efficiently update model parameters and achieve comparable performance to model retraining. In addition, we also conduct a case study on the synthetic data, and carry out some specific analysis for each part of our model, illustrating its ability to learn new knowledge and maintain existing knowledge from different perspectives. 
\end{abstract}


\maketitle

\section{Introduction}

Network structures are widely seen in the real world nowadays. As a method that maps nodes into vectors and preserves local structures, network embedding plays an important role in many applications such as node classification and recommendation, which can be categorized in two types: SkipGram models \cite{perozzi2014deepwalk} and GNNs \cite{kipf2016semi,hamilton2017inductive}. Among them, GNNs take advantages of deep learning techniques and are capable of combining node features and local structures together, achieving great improvements than traditional models. 
But most of the existing methods are based on the assumption that networks are static. The model is trained on the entire network, and parameters will not be updated after training.

However, network data is usually formed in a streaming fashion and most real-world networks are continuously evolving over time, which are referred as streaming networks. For example, edges are added or removed in social networks, and attributes of nodes may also change over time. The dynamics give rise to some \textit{new patterns} \cite{webb2016characterizing}, referring to the neighborhood information not captured by the GNN, and at the same time, many \textit{existing patterns} that have been learned by the GNN in the previous network still maintain. For example, new research fields may grow up in citation networks and in different years, users may generate distinct types of social circles. In order to get the updated representations, it is necessary to learn both new and existing patterns, but retraining the GNN once the network changes results in high computational complexity. 

There are some existing researches on streaming network embedding specially, focusing on either capturing network evolutionary patterns or efficiently updating node representations. Our model belongs to the latter type, with typical examples including DANE \cite{li2017attributed} and DNE \cite{du2018dynamic}. They design strategies to update node representations in a transductive setting, which lack scalability and cannot be migrated to the GNN architecture. 
There are some naive incremental learning approaches for GNNs. On one hand, GNNs can efficiently generate new representations by aggregating features from neighborhood based on an inductive framework but the performance will gradually degrade as parameters do not absorb knowledge from current networks. On the other hand, applying traditional online learning approaches causes \textit{catastrophic forgetting} problem. That is, existing patterns will be overwritten and forgotten abruptly when new patterns are incorporated, resulting in the inability to obtain good representations for previous nodes through the current GNN. 
Therefore, training GNNs incrementally to ensure that models are updated to learn new patterns and consolidate existing patterns simultaneously during network evolution is meaningful.

In supervised learning tasks, data is sampled identically and independently from a fixed distribution, and its learning goal is to minimize the empirical loss for all data. 
\textit{Continual Learning}, also called as life-long learning, is a learning method in a streaming scenario. In the real world, data in different distributions and from different domains arrives in a streaming manner. Storing all data and retraining on each time step is not feasible on large-scale data. Thus, the learning goal of continual learning is to minimize the loss of data under the current task without interfering the data that has been learned before. When the data distribution shifts over time, it can avoid the catastrophic forgetting of previous tasks.  
Hence, we adopt continual learning to address the problem for training GNNs incrementally because network distributions may shift in streaming networks. And our goal is to develop an efficient algorithm that can not only capture new patterns from current networks but also consolidate existing patterns from previous networks.
Yet, the following challenges still remain:
\begin{itemize}
    \item New patterns need to be detected. Due to the complexity of network structures, it is difficult to estimate whether a node is seriously affected and corresponds to a new pattern. Besides, any node of the entire network may be affected, but calculating the influenced degrees of all nodes is time-consuming and not feasible on the large-scale networks.
    \item Existing patterns need to be consolidated. It is unrealistic to save and retrain the entire network, and an alternative is to sample and replay only a small number of nodes. But on one hand, naive sampling strategies are limited by the instability and poor effectiveness. On the other hand, it may lead to the overfitting problem and lacks additional constraints to improve generalization of knowledge maintaining.
\end{itemize}

Given these challenges, we propose a streaming GNN model via continual learning in streaming networks. When distributions of networks shift over time, our model can capture new patterns incrementally and consolidate existing information. 
Firstly, in order to learn new patterns, we propose an approximate method based on the propagation process of changes in the network to mine affected nodes efficiently. 
Then the existing knowledge of the network is modeled and maintained from two complementary perspectives. To improve the effectiveness and stability of data replaying, we design a hierarchy-importance sampling strategy to store some nodes in the memory. At the same time, in order to solve the overfitting problem caused by the small number of replayed nodes, a weighted regularization method for model parameters based on Fisher information is derived. 
We evaluate our model on the real-world and synthetic graphs. Our model outperforms other incremental baselines and achieves lower time complexity than retraining models. In addition, we show that our model can alleviate the problem of catastrophic forgetting through a case study on synthetic data. Each part of the model is also analyzed in detail, further illustrating the superiority of our model in detecting and learning new patterns and consolidating existing patterns.

Overall, our contributions are as follows:
\begin{itemize}
    \item We propose a streaming GNN framework via continual learning that can update model parameters efficiently and generate up-to-date node representations when the network evolves and its distribution shifts over time.
    \item We design an approximation algorithm to detect new patterns. A new data sampling method and a model regularization method are proposed and combined to consolidate the existing patterns in the streaming network.
    \item We conduct experiments on four data sets which show that our streaming GNNs can be updated incrementally with less accuracy lost than other models.
\end{itemize}

\section{Related Work}

\subsection{Network Embedding}
Traditional network embedding methods can be categorized into two types. SkipGram models such as DeepWalk \cite{perozzi2014deepwalk} and LINE \cite{tang2015line} are limited to their transductive setting. GNN models such as GCN \cite{kipf2016semi}, GraphSAGE \cite{hamilton2017inductive}, GAT \cite{velivckovic2017graph} and GIN \cite{xu2018powerful} learn node representations by aggregating neighborhood features at each layer and often achieve outstanding performance in tasks. But all the aforementioned methods are only designed for static networks.

As real-world graph data often appears streamingly, existing works in dynamic networks mainly focus on two aspects. 
Some of them capture temporal information along with local structure so as to improve the expressive ability of the model, with typical examples like DynamicTriad \cite{zhou2018dynamic}, HTNE \cite{zuo2018embedding}, EPNE \cite{wang2020epne}, EvolveGCN \cite{pareja2019evolvegcn}, DynamicGCN \cite{deng2019learning}, DGNN \cite{ma2018streaming} and ST-GCN \cite{yan2018spatial}. But these models have the shortcoming of high computation complexity as they need to train a new model at each time step. 
Other methods study how to efficiently update node representations when networks change. 
DANE \cite{li2017attributed} proposes online learning for dynamic attributed networks with a spectral embedding manner. DNE \cite{du2018dynamic} and \cite{liu2018streaming} consider that it is unnecessary to update representations of all nodes but only the most affected nodes need to be adjusted during the evolvement of networks. But they are all transductive models and cannot transfer to GNNs. 
As traditional GNNs are inductive models, new representations can be obtained directly but the scalability of models is limited as parameters are fixed after training. 

\subsection{Continual Learning}
The key point of continual learning is to consolidate the knowledge learned from previous data in a streaming scenario to avoid catastrophic forgetting. Recent methods can be divided into three categories. A first family are the regularization-based methods by a parameter regularization, including EWC \cite{kirkpatrick2017overcoming}, SI \cite{zenke2017continual}, MAS \cite{aljundi2018memory} and VCL \cite{nguyen2017variational}. Their implementation is simple and enjoy the beauty of the Bayesian framework. A second family are the replay-based methods that old examples reproduced from a replay buffer such as GEM \cite{lopez2017gradient} and iCaRL \cite{rebuffi2017icarl}, or a generative model \cite{shin2017continual} will be used for retraining. The idea of replaying data was proposed very early, but it is still widely used and performs very well in practical applications. The last family are parameter isolation that isolates parameters from different tasks explicitly \cite{mallya2018packnet} and draws more on the idea of transfer learning. All these methods follow the task-based setting, that is, the clear task boundaries are utilized to help store or process data or models from different previous tasks. 

However, task boundaries are often not available in real-world tasks in streaming scenarios so that the above methods are no longer applicable.
\cite{aljundi2019task} and \cite{aljundi2019online} both propose methods that extend continual learning to online setting on streaming data, which are based on regularization and replay respectively. But they only work on independent data like images but cannot deal with streaming data on complex structures like networks.
\section{Preliminaries}
\label{sec:3}
In this section, we introduce the basic idea of GNNs and the problem definition of streaming GNNs, and then give a few simple but insufficient algorithms, which will be used as comparison methods in our experiments.

Let $G=(\mathbf{A}, \mathbf{X})$ be an attributed graph where $\mathbf{A}_{n \times n} \in \{0,1\}$ is the adjacency matrix and $\mathbf{X}_{n \times d} \in [0,1]$ is the node features. Let $V = \{v_i\}_{i=1}^n$ represent the set of nodes and $E = \{e_{ij}\}, i,j \in V$ represents the set of edges. 

\textbf{Graph Neural Networks (GNNs)} are a general architecture for neural networks on graphs to generate node representations by an aggregator for neighborhood information. Here, we follow the established work GraphSAGE \cite{hamilton2017inductive} where the layer $l$-th is defined as:
\begin{equation}
    \mathbf{h}_v^l = \sigma( \mathbf{W}^l \cdot MEAN(\{\mathbf{h}_v^{l-1}\} \cup \{\mathbf{h}_u^{l-1}, \forall u \in \mathcal{N}(v)\}) ),
\end{equation}
where $\mathcal{N}(v)$ specifies neighborhood of node $v$ and $\sigma$ is the activation function. We then use node representations in cross-entropy loss for node classification as
\begin{equation}
    L(\mathbf{\theta}; \mathbf{A}, \mathbf{X}) =  \frac{1}{n} \sum_{v \in V_{label}} l(\mathbf{\theta}; v) = \frac{1}{n} \sum_{v \in V_{label}} l(softmax(\mathbf{h}_v^L), y_v),
\end{equation}
where $V_{label} \subset V$ is a subset of labeled nodes, $y_v$ is the label of node $v$ and $L$ is the number of layers of the GNN.

A \textbf{Streaming Network} is represented as $G = (G^1, G^2, ..., G^T)$, where $G^t = G^{t-1} + \Delta G^t$. 
$G^t = (\mathbf{A}^t, \mathbf{X}^t)$ is an attributed graph at time $t$ and  
$\Delta G^t = (\Delta \mathbf{A}^t, \Delta \mathbf{X}^t)$ is the changes of node attributes and network structures at $t$. Notice that the changes contain newly added nodes, which are difficult to handle for many existing dynamic network representation learning methods.

\textbf{Streaming Graph Neural Networks (Streaming GNNs)} are extensions of traditional GNNs in a streaming setting. Given a streaming network $G$, the goal is to learn $(\mathbf{\theta}^1, \mathbf{\theta}^2, ..., \mathbf{\theta}^T)$ where $\mathbf{\theta}^t$ is the parameters of the GNN at $t$ that can generate satisfactory $h_v^L$ representation for any node $v \in G^t$. 

In the streaming graph data, the network is continuously evolving, resulting in changes in the patterns, referring to the neighborhood information of the node captured by GNNs, including network structures and node attributes. The changes of these patterns result in that node representations obtained through GNNs need to be updated. We discuss three simple schemes as follows.
\begin{itemize}
    \item \textit{Pretrained GNN}: As GNNs are inductive models, a straightforward idea to generate representations for the previously unseen nodes and changed nodes on $G^t$ is to use a pre-trained GNN model. But if the patterns of these nodes are different from knowledge in the pre-trained model, then good representations cannot be obtained.
    \item \textit{Retrained GNN}: Another naive solution is to train a new GNN on the whole graph data over again. It may has high performance but suffers from affordable time and space cost if the network becomes large as time goes by.
    \item \textit{Online GNN}: A better solution is to train on $\Delta G^t$ using $\theta^{t-1}$ to initialize, which is called online learning. But if the patterns in $\Delta G^t$ are different from $\theta^{t-1}$, catastrophic forgetting will occur. Knowledge of $\theta^{t-1}$ may be abruptly lost and degraded representations of nodes in $G^{t-1}$ will be obtained.
\end{itemize}

It can be seen that the above schemes all have great limitations when dealing with large-scale streaming graph data whose pattern distribution shifts over time. So in the next section, we propose our streaming GNN model to avoid all the deficiencies of the above schemes.

\section{Streaming GNNs via continual learning}

In this section, we first introduce the basic framework of our model, called \textbf{ContinualGNN}, which can learn new patterns that appear in the network and consolidate existing patterns simultaneously. Next, we propose an efficient approximation algorithm to detect the emergence of new patterns in streaming networks, and a multi-view based method to consolidate existing patterns.

\subsection{Model Framework}

Due to the limitations of all naive methods as discussed in Section \ref{sec:3}, 
we propose streaming GNNs via continual learning that has been proved to be effective and memory efficient to mitigate catastrophic forgetting. That is to say, through continual learning, our model can maintain the knowledge of previous tasks and current tasks at the same time in a streaming scenario. 
In the followings, we give some concepts corresponding to continual learning in the context of streaming networks, and then on this basis, we present the general goal of our model.

\begin{itemize}
\item \textit{Current Tasks $D^t$}: A current task is defined as training a GNN with parameters $\theta^t$ on the new network $\Delta G^t$. But as nodes and edges are not independent in networks, $\Delta G^t$ may influence the neighborhood of other nodes. It is more proper to define the current task as affected parts of the network.
\item \textit{Previous Tasks $(D^1,...,D^{t-1})$}: Previous tasks are defined as training GNNs on the previous changes of the network $(\Delta G^1, \Delta G^2, ..., \Delta G^{t-1})$. 
If there are lots of time steps in streaming data, it will be unaffordable to preserve information for each previous task at each snapshot. So we replace streaming data with an integrated network $G^{t-1}$ so that only a single task will be considered for knowledge consolidation. 
\end{itemize}

Furthermore, following \cite{kirkpatrick2017overcoming}, we explain our goal from a probabilistic perspective. Since we want to learn a GNN parameterized by $\theta$ on $G^t = G^{t-1} + \Delta G^t$ with a supervised loss, we use $\Delta G^t$ and $G^t$ to represent the nodes and the corresponding labels to simplify formulas and then the conditional probability can be computed by using Bayes' rule:
\begin{equation}
\begin{aligned}
    \log p(\theta|G^{t-1}, \Delta G^t) = \log p(\Delta G^t|\theta) + \log p(\theta|G^{t-1}) - \log p(\Delta G^t),
\end{aligned}
\end{equation}
where the first term $\log p(\Delta G^t|\theta)$ is the log likelihood for the task on $\Delta G^t$. The second term $\log p(\theta|G^{t-1})$ is the posterior distribution that absorbs information learned from the task on $G^{t-1}$. The last term $\log p(\Delta G^t)$ is a constant.
Then we have the general loss function of our ContinualGNN in streaming network at time $t$ as
\begin{equation}
\begin{aligned}
L = L_{new} + L_{existing}
\end{aligned}
\label{equ:loss}
\end{equation}
where the first term $L_{new}$ is the loss function on the influenced parts of networks. The second term $L_{existing}$ aims to consolidate patterns on previous network data. The following sections discuss how to design these two loss functions separately.

\subsection{Detection for New Patterns}
\label{sec:detection}

In a streaming scenario, network changes lead to the emergence of new patterns that was not captured by the previous GNN. The GNN need to be updated, which corresponds to the current task in continual learning. But it is insufficient to consider that the current task is to train the nodes in $\Delta G^t$. On the one hand, these nodes may not contain new patterns if the changes are small. For example, when the neighborhood of a node is very stable, adding an edge will not have a great impact on its representation and retraining these nodes is a waste of time. On the other hand, new patterns may still appear in existing nodes. Considering that the network data is not independent of each other, some changes may affect other parts of the network. Any node may have new pattern during the network change, and these nodes also need to be retrained. So we should study how to mine the nodes that may have new patterns.

First of all, in order to solve the above problem, we give the definition of new patterns related to the influenced degree of nodes. If nodes are greatly affected, they may contain new patterns. The influenced degree of a node is defined by a \textit{scoring function} which is the delta of representations obtained from the GNN between two adjacent snapshots. Then the set of influenced nodes, corresponding to the new patterns, is generated according to the scores as
\begin{equation}
\begin{aligned}
\label{equ:influenced_node_set}
    \mathcal{I}(\Delta G^t) = \{ u | \|\Delta \mathbf{h}_u^{t, L}\| = \|\mathbf{h}_u^{t + \Delta t, L} - \mathbf{h}_u^{t, L}\| > \delta \},
\end{aligned}
\end{equation}
where $\delta$ is a threshold value and a hyper-parameter. A smaller $\delta$ allows more nodes to be treated as new patterns, and vice versa. 
However, it is time-consuming to calculate the scoring function for all nodes at each time step. In fact, this is not necessary because only the neighbors within the order $L$ are affected. An improved idea is to use breadth-first-search (BFS) to find all the $L$-order neighbors and calculate their scoring function, but the complexity is still related to the size of the neighborhood of $\Delta G^t$. 

Secondly, we design a fast approximation algorithm for the scoring function based on the information propagation of the changes, and the time complexity is only related to $\Delta G^t$. In the following, we first present a lemma, and then based on the lemma, give and analyze our approximate algorithm.

\begin{lemma}
    \label{lem:1}
    When \textit{attributes} of node $v_i$ are changed from $\mathbf{x}_i$ into $\mathbf{x}_i + \Delta \mathbf{x}_i$, the score of node $v_u$ can be calculated as:
\begin{equation}
\begin{aligned}
\label{equ:h}
    \Delta \mathbf{h}_u^{t,L}  = f^{t, L}_u (\Delta \mathbf{x}_i \mathbf{\tilde{W}}),
\end{aligned}
\end{equation}
    where $\mathbf{\tilde{W}} = \prod_{l=1}^L \mathbf{W}^l$ and $f^{t,l}_u$ is a function calculated as
\begin{equation}
\begin{aligned}
\label{equ:f}
    & f^{t,l}_u = \frac{1}{d_u^t}(\sum_{u' \in N^t(u)} f^{t, l-1}_{u'}), l > 0\\
    & f^{t, 0}_i = 1, f^{t, 0}_u = 0, u \ne i
\end{aligned}
\end{equation}
where $d_u^t$ and $N^t(u)$ are the degree and neighbors of $v_u$ at time $t$.
\end{lemma}

\begin{proof}
    Following \cite{zugner2018adversarial}, we can use a surrogate model that still captures the idea of convolutions to accelerate the computation of scoring function. The non-linear activation function is replaced with a linear function so that parameters $\mathbf{W}^l$ can be absorbed into a matrix $\mathbf{\tilde{W}}$ that equals to the production of $\mathbf{W}^l$ at each layer. Then, $\Delta \mathbf{h}_u^{t,L}$ can be calculated as
    \begin{equation}
    \begin{aligned}
    \Delta \mathbf{h}_u^{t,L} & = \mathbf{h}_u^{t + \Delta t, L} - \mathbf{h}_u^{t, L} \\
    & = \frac{1}{d^t_u}(\sum_{u' \in N^t(u)} \Delta \mathbf{h}_{u'}^{t, L-1}) \mathbf{W}^l \\
    & = ... = f^{t, L}_u \Delta \mathbf{x}_i \mathbf{\tilde{W}},
    \end{aligned}
    \end{equation}
    where $f^{t,l}_u$ reflects the influence of node $v_i$ on $v_u$ by propagating information through the network and is only related to node degrees. Node $v_i$ will have a greater influence on $v_u$ if there are more paths from $v_i$ to $v_u$, and the degrees of other nodes on these paths are smaller.
\end{proof}

According to Lemma \ref{lem:1}, we discuss how to estimate the influenced degree of nodes efficiently from two perspectives. 
\begin{itemize}
    \item We first discuss the method of calculating $\Delta \mathbf{h}_u^{t,l}$ under the simplest network change, that is, when the attributes of a node change. Compute $\Delta \mathbf{x}_i \mathbf{\tilde{W}}$ once. Then implement BFS from node $v_i$ and compute $f^{t,l}_u$ iteratively, which is only related to degrees of nodes at time $t$. The search will iterate over $L$ rounds. Finally, after $\Delta \mathbf{x}_i \mathbf{\tilde{W}}$ and $f^{t,L}_u$ are obtained, $\Delta \mathbf{h}_u^{t,L}$ can be calculated according to Equation \ref{equ:h}.
    \item We then consider the more general and complex situation in the network, that is, when the structures related to more than one node change. In this case, we adopt an approximate idea that the structural changes are only considered once during the information propagation process of GNNs. From the implementation point of view, $\Delta \mathbf {h}_i^1$ is computed and regarded as $\Delta \mathbf {x}_i$, which is converted to the situation where the node attribute changes. Then use a similar method as before to start a BFS from the changed node set, iteratively calculate $f_{u,i}^{t,l}$ of each possible affected node $v_u$ affected by node $v_i$, and iterate $L$ rounds. In the end, the influenced degree of each node $v_u$ is quickly approximated, which is the summation of the influenced degree of each $v_i$:
    \begin{equation}
    \begin{aligned}
    \label{equ:approximation}
        \Delta \mathbf{h}_u^{t,L} \approx \sum_{v_i \in \Delta V^t} f_{u,i}^{t,L} \|\Delta \mathbf{h}_i^{t,1}\|,
    \end{aligned}
    \end{equation}
    where $\Delta V^t$ is denoted as the node set of $\Delta G^t$. 
\end{itemize}

The complexity of the original methods is $O(|V| \times \tilde{m})$ where $\tilde{m}$ is the complexity to compute $\mathbf{\tilde{W}}$, while the complexity of our approximation algorithm is reduced to $O(|\Delta V| \times \tilde{m} + |N_{\Delta V}^L| \times |\Delta V|)$, where $|N_{\Delta V}^L|$ is the number of $L$-hop neighbors of the node set $\Delta V$.

Finally, according to the above definition and approximate method, we obtain the set of influenced node $\mathcal{I}(\Delta G^t)$ efficiently and then have the objective that replays these nodes to learn new patterns for the GNN at time $t$:
\begin{equation}
    L_{new} = \sum_{v_i \in \mathcal{I}(\Delta G^t)} l(\theta;v_i).
\end{equation}

\subsection{Preservation for Existing Patterns}
\label{sec:preservation}

In a streaming network, when the GNN learns new knowledge, patterns in the historical tasks needs to be consolidated so as to avoid catastrophic forgetting. In the field of continual learning, one way to consolidate historical knowledge is to replay samples. But replaying all nodes at each snapshot is time-consuming. An alternative that a small number of nodes are saved and replayed is prone to overfitting problems. Hence, we propose a method to consolidate existing patterns of the network from two perspectives of data and model, thereby reducing the training complexity and improving the preservation of existing patterns. 

\subsubsection{Data-view}

In order to maintain existing patterns, we can save some historical graph data, including nodes and their neighbors, and then revisit these data at each step of the incremental training. Suppose that we have a memory $\mathcal{M}$ to store some historical data. 
The simple random sampling strategy often has relatively large errors and unstable results. Besides, some studies \cite{schaul2015prioritized,yin2017knowledge} have pointed out that the importance of different samples is different. The unimportant samples do not contribute much to the model convergence and may fail to consolidate the historical knowledge well. Therefore, we propose a step-wise sampling strategy. On the premise of ensuring the stable distribution of memory categories, select more important nodes to save, and revisit them during the process of training in the streaming scenario.

\begin{itemize}
    \item \textit{Hierarchical sampling on clusters}: We divide the nodes in the network into several categories, which are labels or communities of nodes. The probability of the $k$-th node being sampled is $p_k(v_i)=\frac{m_k}{n_k}=p$. 
    On the one hand, it can ensure that the distribution of different types of nodes in memory and in the entire network is consistent. On the other hand, its sampling error is smaller than simple random sampling: $\sigma_{hir}(\mathcal{M}) \le \sigma_{hir}(\mathcal{M}) + \frac{1-p}{m} \sum \frac{n_k}{n-1}(\bar{x}_k - \bar{x})^2 = \sigma_{ran}(\mathcal{M})$. In addition, it can ensure that at least one node of each type of node is stored in memory, which is more suitable for scenes with imbalanced categories in historical samples.
    \item \textit{Importance-based sampling within clusters}: In hierarchical sampling, the importance of each node in a cluster is the same, which is not reasonable enough, we further propose to sample each type of nodes according to the importance of the nodes. The importance of a node $v_i$ is designed as $\frac{1}{|\mathcal{N}_i|}\sum_{v_j \in \mathcal{N}_i} \mathbf{I}(y_i \ne y_j)$, where $\mathbf{I}$ is the indicator function. A node is important if its attributes are quite different from the attributes of its neighbors. These important nodes are more likely to locate at the class boundary and contribute more to the gradient. Revisiting them can better consolidate the historical knowledge. 
\end{itemize}

In the streaming scenario, we adopt the reservoir sampling algorithm \cite{vitter1985random} to extend the step-wise sampling strategy so as to update the memory $\mathcal{M}$ online, and the probability that node $v_i$ replaces the $k$-th node in the memory is
\begin{equation}
\label{equ:mem}
    p(v_i) = p_k(v_i) = \frac{m_k}{n_k} [1 + \alpha \cdot \frac{1}{|\mathcal{N}_i|}\sum_{v_j \in \mathcal{N}_i} \mathbf{I}(y_i \ne y_j)],
\end{equation}
where $\alpha$ is a hyper-parameter that controls the weight of importance-based sampling strategy. Then, given the memory $\mathcal{M}$, we get the optimization goal of consolidating historical knowledge from the perspective of data replaying as
\begin{equation}
    L_{data} = \sum_{v_i \in \mathcal{M}} l(\theta;v_i).
\end{equation}

\subsubsection{Model-view}

Only saving a part of historical data, and replaying these data, because this part of the data is relatively small, is prone to the overfitting problem. With the help of saving the model, the historical model is used to constrain the current model, so as to alleviate the phenomenon of overfitting, and further improve the generalization of knowledge preservation. Using L2 regularization constraints directly on the model parameters is an overly rough method and may make the results worse since it will lead to over-smoothing or invalid constraints. 
So, following Elastic weight consolidation (EWC) \cite{kirkpatrick2017overcoming}, a regularization-based method in continual learning, we model the posterior distribution $p(\theta|G^{t-1})$ and derive that it is equivalent to imposing weighted regularization on parameters of the GNN model. 

As $p(\theta|G^{t-1})$ is intractable, the Laplace approximation is adopted to approximate the posterior. Let $q(\theta) = p(\theta|G^{t-1})$ and assume the posterior of GNN parameters to follow a Gaussian distribution with mean given by the parameters $\theta^{t-1}$. Since $\log q(\theta)$ is a quadratic function, we carry out Taylor Expansion on $\theta^{t-1}$ and obtain
\begin{equation}
    \log q(\theta) = -\frac{\mathbf{F}}{2} (\theta - \theta^{t-1})^2.
\end{equation}

So we have $p(\theta|G^{t-1}) \sim N(\theta^{t-1}, 1/\mathbf{F})$, where $\mathbf{F}$ is the Fisher Information matrix and can be computed from first-order derivatives. However, calculating $\mathbf{F}$ at each step requires traversing the entire network, and the cost of storage and calculation is high. Fortunately, we use the nodes and their neighbors in memory $\mathcal{M}$ mentioned above to estimate $\mathbf{F}$ of GNN parameters as
\begin{equation}
\begin{aligned}
\label{equ:fisher}
\mathbf{F} & = E_{v}[(\frac{\partial \log p(\theta;v)}{\partial \theta})(\frac{\partial \log p(\theta;v)}{\partial \theta})^T] \\
    & = \frac{1}{n} \sum_{v \in G^{t-1}}[g(\theta;v) g(\theta;v)^T] \\
    & \approx \frac{1}{m} \sum_{v \in M}[g(\theta;v)g(\theta;v)^T].
\end{aligned}
\end{equation}

By approximating the posterior, we derive the objective to preserve historical information from model view:
\begin{equation}
    L_{model} = \lambda \sum_i \mathbf{F}_i (\theta_i - \theta_i^{t-1})^2,
\end{equation}
which can be understood as a weighted smoothing term of GNN parameters. $\lambda$ sets how important the previous information is compared to the new network. $\mathbf{F}_i$ indicates the importance of the $i$-th parameter and $\theta_i^{t-1}$ is its optimal value at the last time step. 
On the one hand, the regular term guarantees that the distance between the current model parameters and the historical model parameters will not deviate further. On the other hand, compared with the ordinary L2 smoothing of the parameters, the key of the objective is to add the importance of different parameters, so that the changes of GNN parameters that are important to the past network are small, which guarantees the consolidation of historical information, and unimportant parameters can be updated more drastically.

Combining the above two perspectives, we obtain the optimization goal of consolidating historical knowledge as 
\begin{equation}
    L_{existing} = L_{data} + L_{model}.
\end{equation}

\subsection{Model Optimization}

Combining the above objectives, we obtain the optimization goal of our model at each time step $t$ as
\begin{equation}
\begin{aligned}
\label{equ:loss}
    L & = L_{new} + L_{data} + L_{model} \\
    & = \sum_{v_i \in \mathcal{M} \cup \mathcal{I}(\Delta G^t)} l(\theta;v_i) + \lambda \sum_i F_i (\theta_i - \theta_i^{t-1})^2.
\end{aligned}
\end{equation}

We use Stochastic Gradient Descent (SGD) to optimize the objective so as to obtain the parameters $\theta^t$ of the GNN at each time step $t$. We summarize our algorithm in Algorithm \ref{alg:gnn}. 
Comparing with the time complexity at each time of retraining relative to $O(|V^t|)$, the complexity to minimize Equation \ref{equ:loss} depends on the learning of the changed parts and the consolidation of existing knowledge in networks. The learning of the changing part is related to the number of affected nodes, that is $O(|\mathcal{I}(\Delta G^t)|)$. And for the consolidation of existing knowledge, its complexity $O(m)$ is irrelevant to the size of streaming networks.

\begin{algorithm}
	\caption{Learning Algorithm of ContinualGNN at time $t$}
	\label{alg:gnn}
	\begin{algorithmic}[1]
		\Require Network snapshot at time $t$:  $G^t = \{V^t, E^t\}$, GNN parameterized by $\theta^{t-1}$ learned on $G^{t-1}$
		\Ensure GNN parameterized by $\theta^{t}$ learned on $G^{t}$
        \State Approximate scoring function and obtain influenced node set $\mathcal{I}(\Delta G^t)$ according to Equation \ref{equ:influenced_node_set} and Equation \ref{equ:approximation}
        \State Load important nodes from memory $\mathcal{M}$
        \State Calculate parameter importance $F$ according to Equation \ref{equ:fisher}
        \For{$e=1$ to $num\_epoches$}
            \State Calculate loss function according to Equation \ref{equ:loss} 
            \State Update parameters using SGD
        \EndFor
        \State Update memory $\mathcal{M}$ using influenced node set $\mathcal{I}(\Delta G^t)$ according to Equation \ref{equ:mem}
	\end{algorithmic}
\end{algorithm}
\section{Experiments}
In this section, we evaluate the effectiveness and efficiency of our model from multiple perspectives. Firstly, we perform node classification on four data sets to prove that our model can efficiently implement incremental learning in streaming networks. Secondly, we conduct a case study on a synthetic data to show that our model can handle catastrophic forgetting. Finally, we conduct a detailed analysis of each part of our model.

\subsection{Experimental Setup}

\subsubsection{Datasets}
We conduct experiments on three real-world data sets and a synthetic data set.
\begin{itemize}
    \item \textit{Cora} \cite{mccallum2000automating} is a static citation network with 2708 nodes, 5429 edges, 1433 features and 7 labels. A node in the network is a paper and an edge represents a paper cited or was cited by another paper. Labels represent research fields. A synthetic dynamic network is generated, which consists of 14 time steps whose label distributions vary in different snapshots. 
    \item \textit{Elliptic} \cite{weber2019anti} is a bitcoin transaction network and we extract a sub-network with 31448 nodes, 34230 edges, 166 features, 2 labels and 9 time steps. A node is a transaction and an edge is viewed as a flow of bitcoins between two transactions. A node label indicates whether it is licit. The data has 9 time steps, each of whom is spaced with an interval of two weeks.
    \item \textit{DBLP} \cite{tang2008arnetminer} is a dynamic citation network and we extract a sub-network with 20000 nodes, 75706 edges, 128 features, 9 labels and 24 time steps ranging from 1993 to 2016.
    \item \textit{Synthetic} consists of 3072 nodes, 14788 edges, 64 features, 2 labels and 24 time steps. Different network generation algorithms and different node attribute generation algorithms are utilized to simulate streaming graph data with different patterns over time. 
    We use the ER graph generator \cite{gilbert1959random} and the community graph generator \cite{fortunato2010community}, and generate nodes attributes with different normal distributions as examples. Morw details of simulation are introduced in Section \ref{sec:exp_synthetic}.
\end{itemize}

\begin{table*}[htbp]
\small
\centering
\setlength{\tabcolsep}{4pt}
\begin{tabular}{c|c|ccc|ccc|ccc|ccc}
    \toprule
    \multicolumn{2}{c|}{Dataset} & \multicolumn{3}{c|}{Cora} & \multicolumn{3}{c|}{Elliptic} & \multicolumn{3}{c|}{DBLP} & \multicolumn{3}{c}{Synthetic} \\ 
    \midrule
    \multicolumn{2}{c|}{Metric} & F1 & Accuracy & Time & F1 & Accuracy & Time & F1 & Accuracy & Time & F1 & Accuracy & Time \\
    \midrule
    \multirow{2}{*}{SkipGram} & LINE & 0.3940 & 0.5292 	& 0.1510 & 0.5608 & 0.7366 & 5.3147 & 0.2234 & 0.3293 & 27.487 &	0.6804 	&	0.6829 	& 0.2715 \\
    & DNE & 0.3519 & 0.4947	&	0.0750 	&	0.5385 	&	0.7338 &	0.7772 	&	0.1979  &	0.3138 	& 0.8217 &	0.5420 	&	0.5454 	& 0.0171  \\
    \midrule 
    \multirow{3}{*}{\shortstack{GNNs\\(Retrained)}} & GraphSAGE & \textbf{0.6692}	&	\textbf{0.8443}	&	0.2847 	&	\textbf{0.9273}	&	\textbf{0.9411}	&	0.1920 &	0.6415 	&	\textbf{0.6699}	&	2.6969  &	0.7536 	&	0.7569 	&	0.2756  \\
    & GCN & 0.6554 	&	0.8266 	&	0.2803 	&	0.9246 	&	0.9229	&	0.1816 	&	\textbf{0.6425} 	&	0.6683 	&	2.7628 &	\textbf{0.7536} 	&	\textbf{0.7575} 	&	0.2559  \\
    & EvolveGCN & 0.4416 	&	0.6014 	&	2.7608 	&	0.8449 	&	0.8793 	&	1.7532 	&	0.5679 	&	0.5871 	&	6.6777 &	0.6297 	&	0.6347 	&	1.7725 \\
    \midrule
    \multirow{4}{*}{\shortstack{GNNs\\(Incremental)}} & PretrainedGNN & 0.2812	&	0.5362	&	0.0000	&	0.7338 	&	0.8661	&	0.0000	&	0.1741 	&	0.3512 	&	0.0000 &	0.5517 	&	0.5861 	&	0.0000 \\
    & SingleGNN & 0.2507	&	0.4186	&	0.0281 	&	0.7479	&	0.7754	&	0.0267	&	0.3038	&	0.4499 	&	0.5408 &	0.6445 	&	0.6498 	&	0.0375 \\
    & OnlineGNN & 0.6018	&	0.7724	&	0.0290 	&	0.7631	&	0.7911	&	0.0286 &	0.6171	&	0.6467 &	0.5304 &	0.6592 	&	0.6644 	&	0.0418 \\
    \cmidrule{2-14}
    & ContinualGNN & \textbf{0.6496}	&	\textbf{0.8245}	&	0.0366	&	\textbf{0.9035}	&	\textbf{0.9212}	&	0.0641	&	\textbf{0.6294} 	&	\textbf{0.6685} 	&	0.4952 &	\textbf{0.7281} 	&	\textbf{0.7337} 	&	0.0884 	\\
    \bottomrule
\end{tabular}
\caption{Averaged F1, Accuracy and Running Time per Epoch for Node Classification}
\label{tab:nc}
\end{table*}

\begin{figure*}[!h]
\centering
\subfigure[Cora]{
    \includegraphics[width=1\columnwidth]{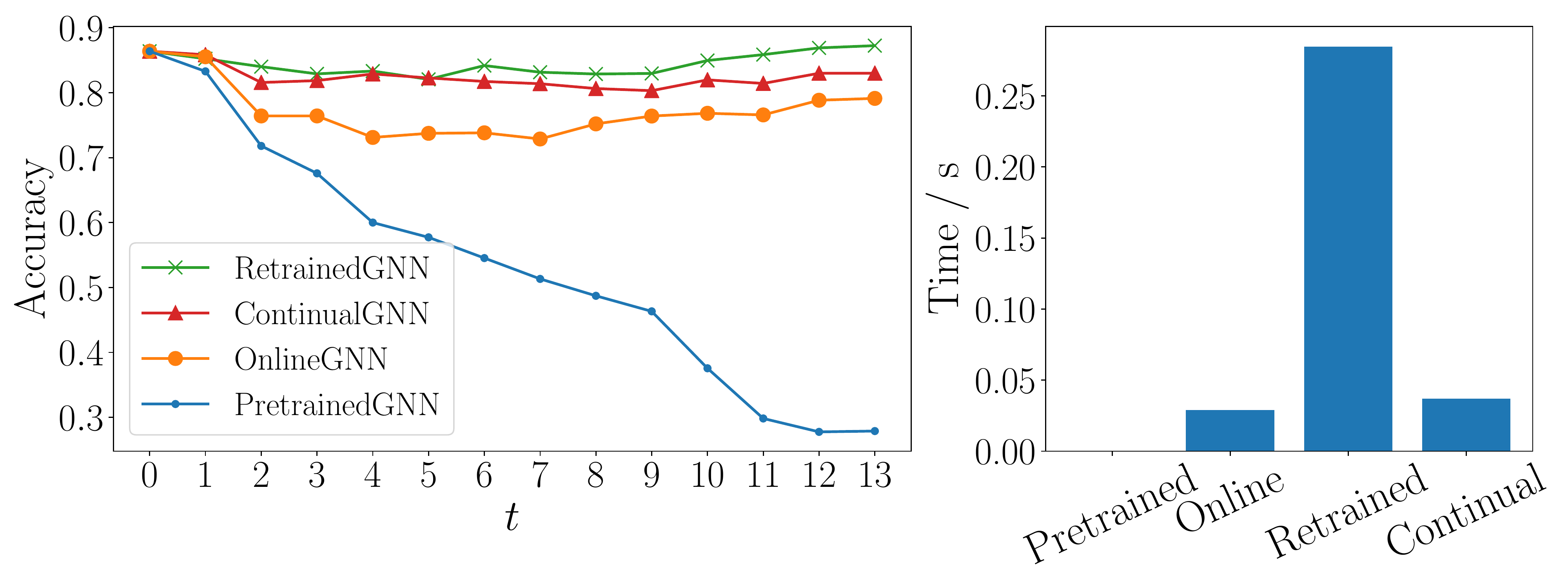}}
\hspace{0.2cm}
\subfigure[Elliptic]{
    \includegraphics[width=1\columnwidth]{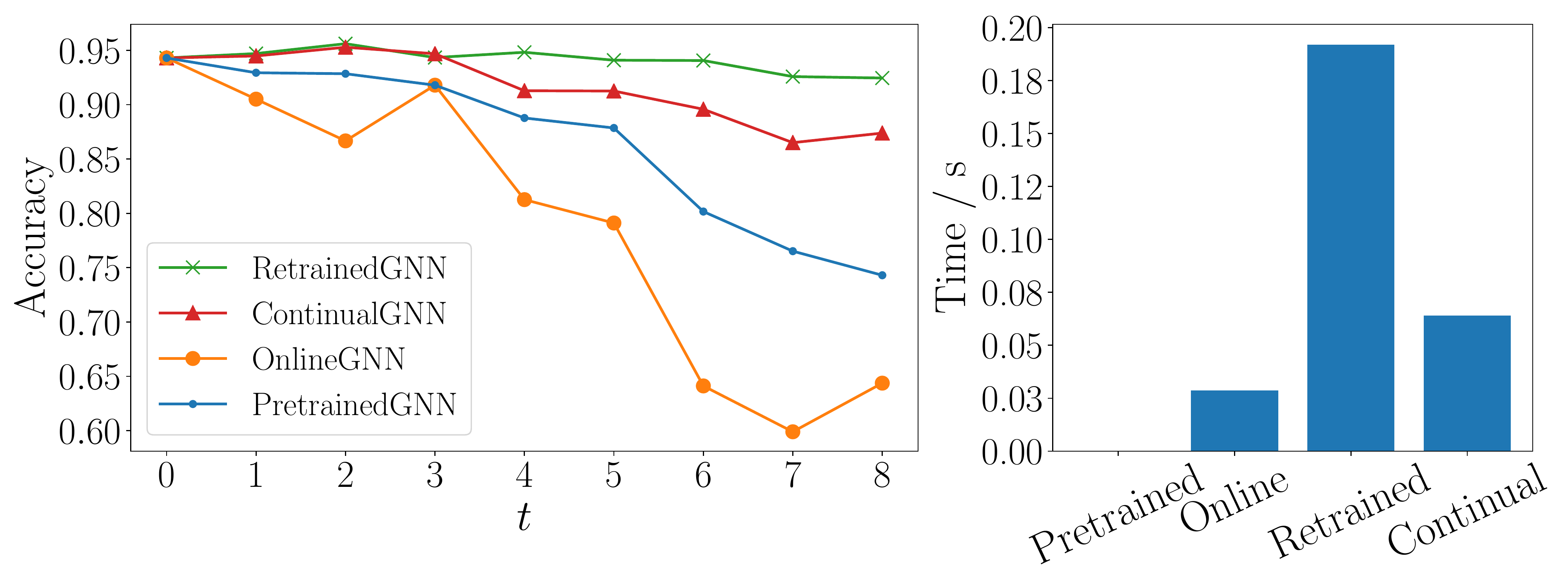}}
\subfigure[DBLP]{
    \includegraphics[width=1\columnwidth]{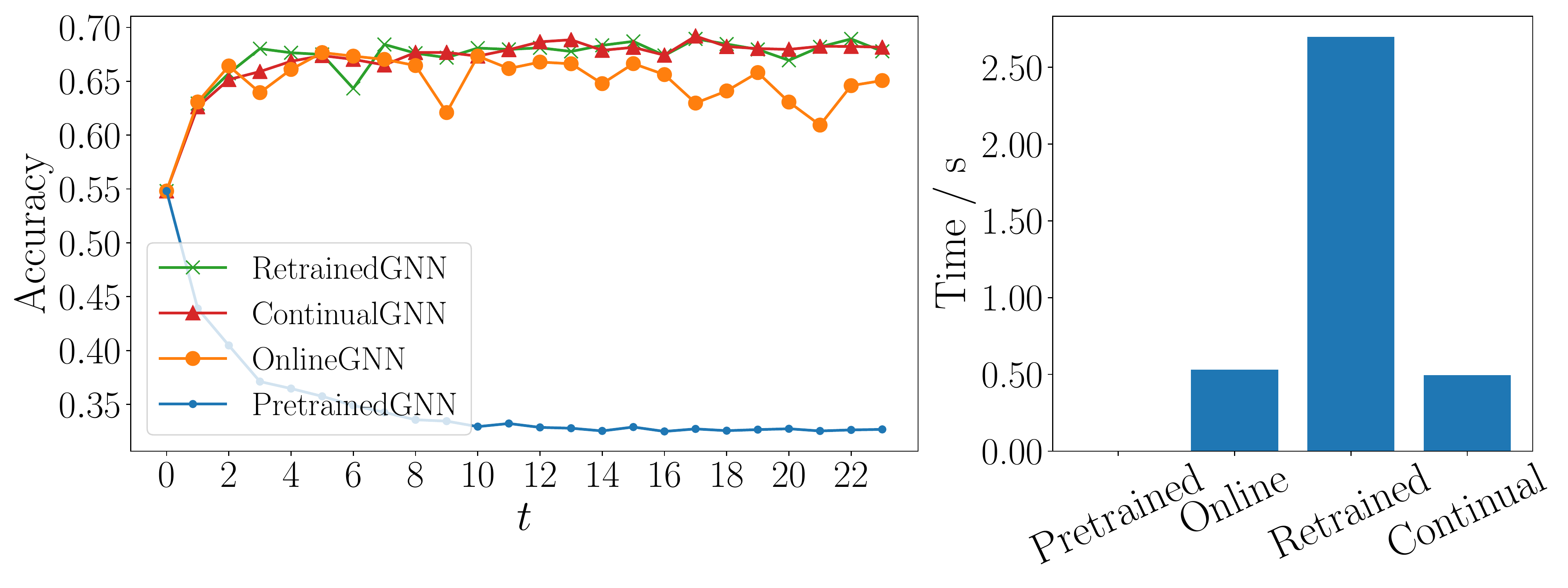}}
\hspace{0.2cm}
\subfigure[Synthetic]{
    \includegraphics[width=1\columnwidth]{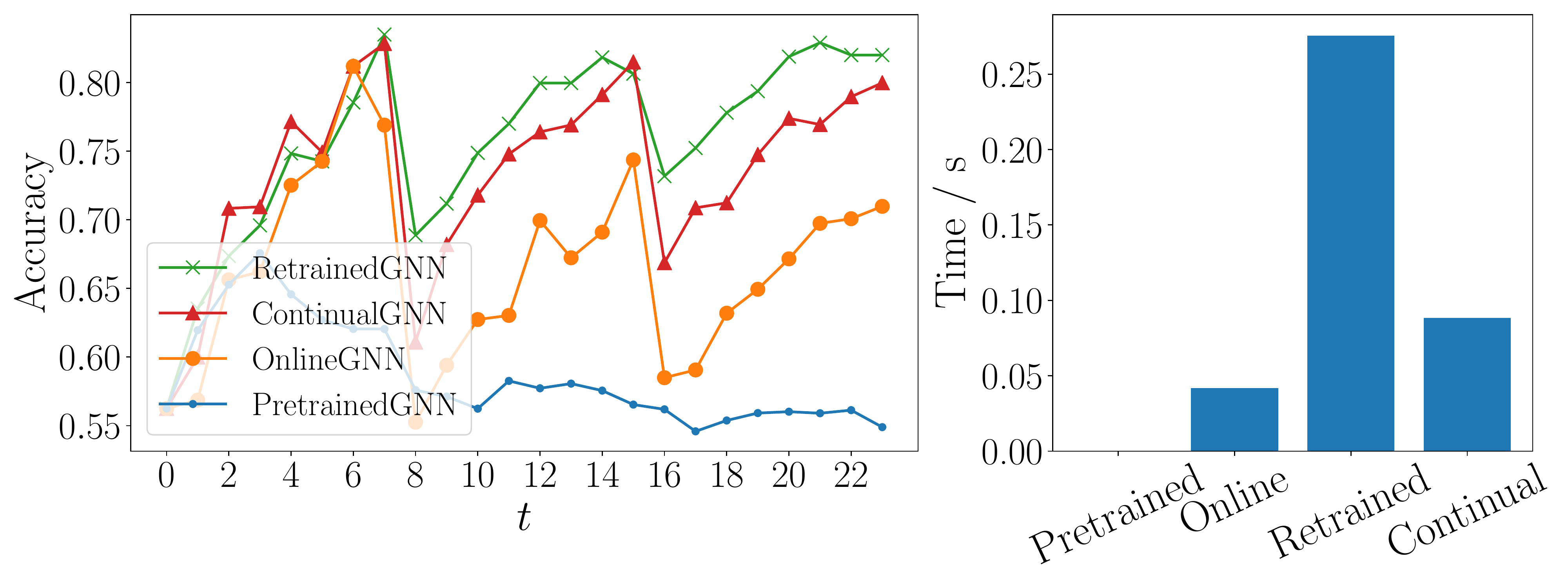}}
\caption{Node Classification on Consecutive Snapshots}
\label{fig:nc}
\end{figure*}

\begin{figure*}[!htbp]
\centering
\begin{minipage}[b]{1\linewidth}
    \centering
    \subfigure[Node classification results of nodes in $V^0$]{
        \label{fig:acc_v0}
        \includegraphics[width=0.29\columnwidth]{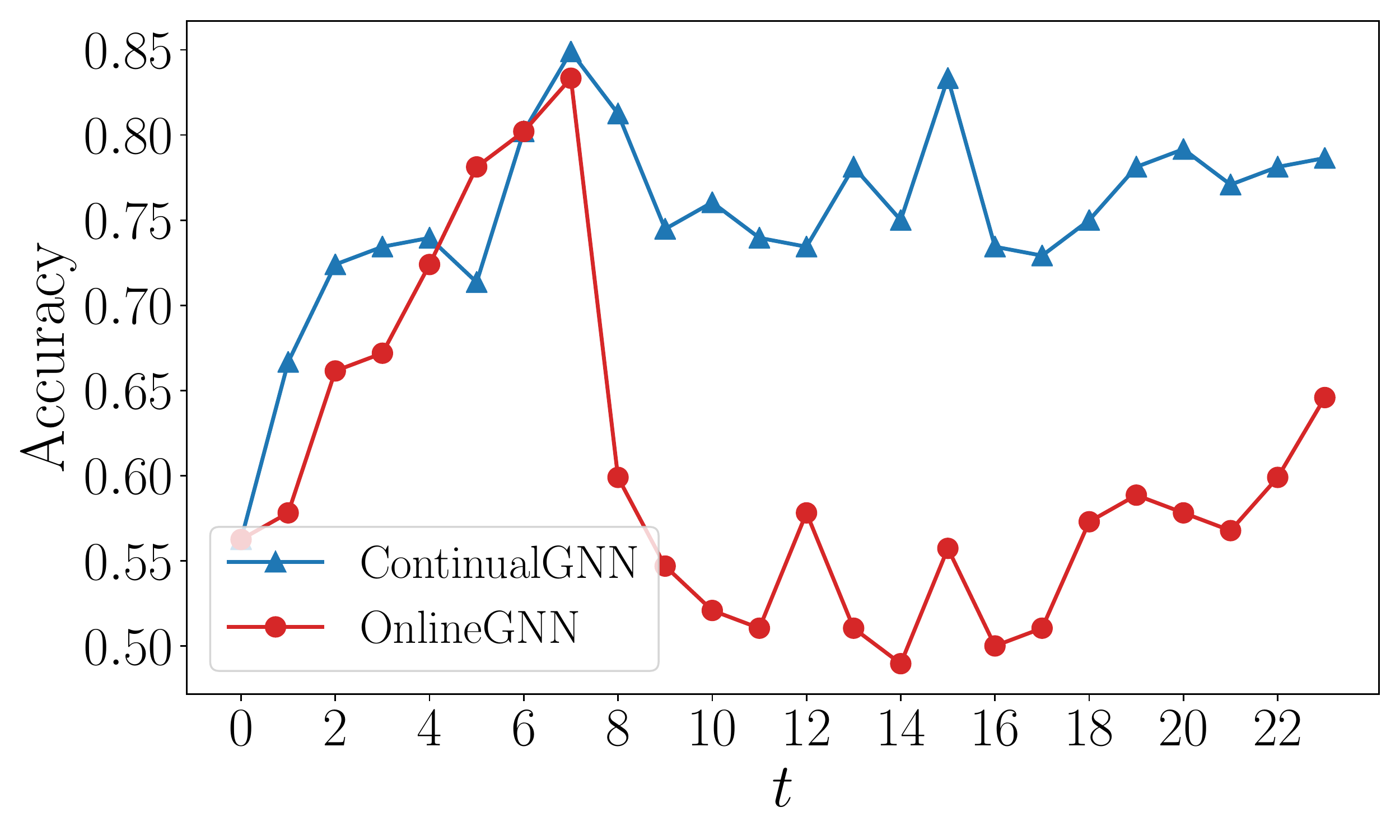}}
    \subfigure[Visualization of nodes in $V^0$ at time $t_7$ and $t_8$]{
        \label{fig:vis_v0}
        \includegraphics[width=0.66\columnwidth]{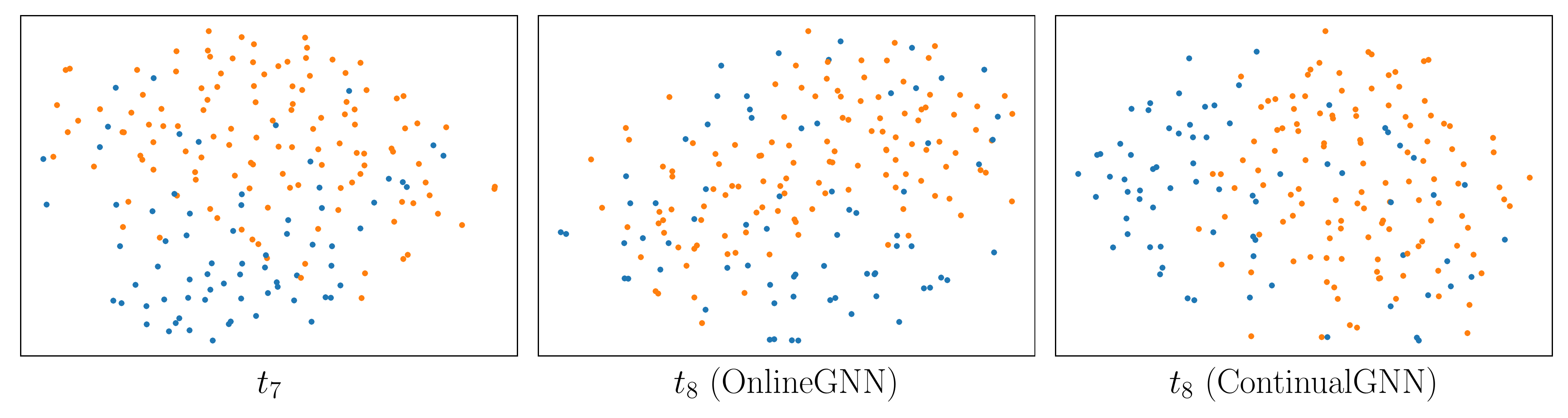}}

    \subfigure[Node classification results of nodes in $V^8$]{
        \label{fig:acc_v8}
        \includegraphics[width=0.29\columnwidth]{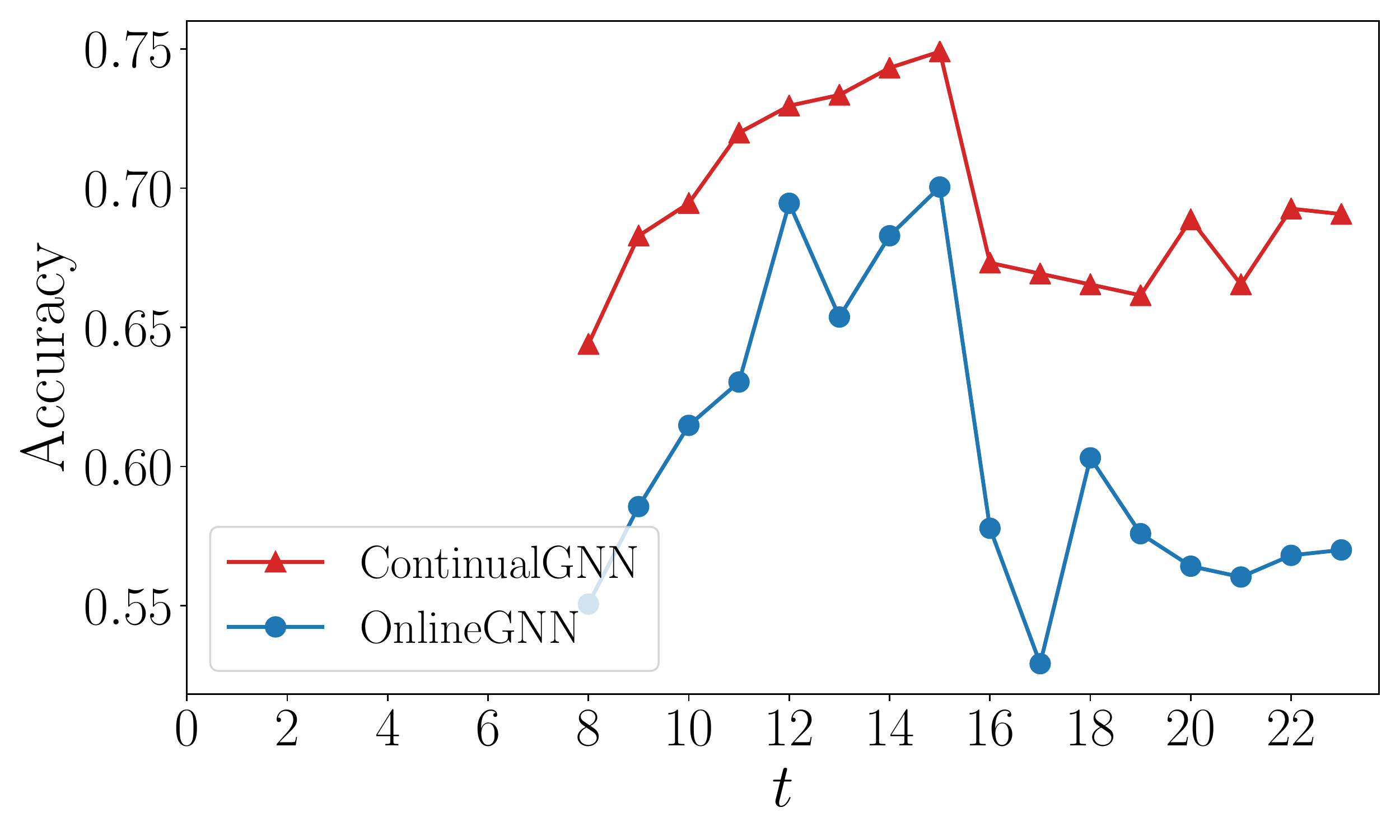}}
    \subfigure[Visualization of nodes in $V^8$ at time $t_{15}$ and $t_{16}$]{
        \label{fig:vis_v8}
        \includegraphics[width=0.66\columnwidth]{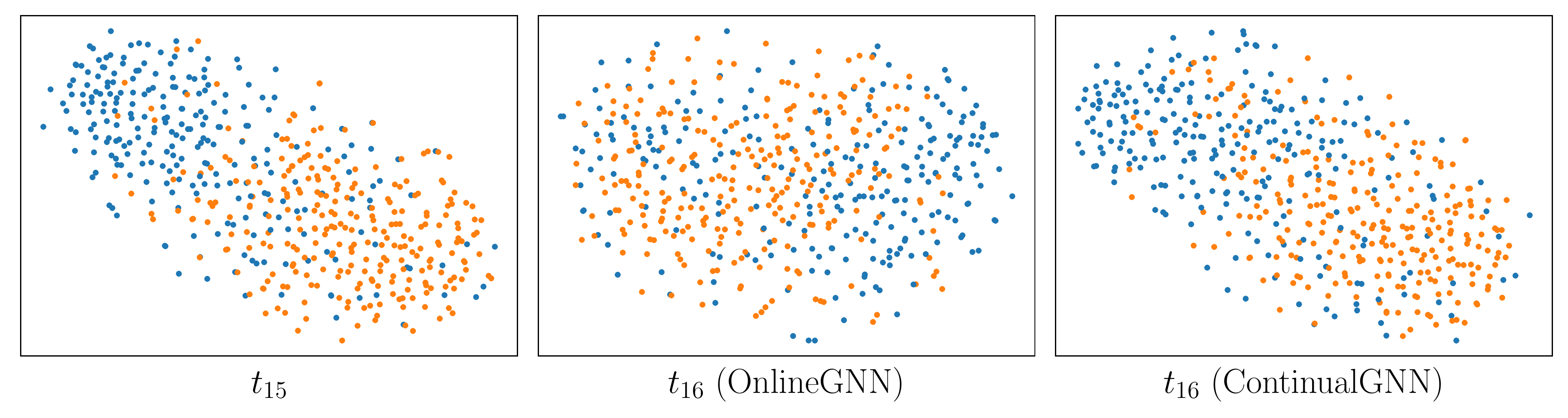}}
    \caption{Case Study on Synthetic Data}
\end{minipage}
\end{figure*}

\subsubsection{Baselines}
We compare our model with totally 8 baselines. 
\begin{itemize}
    \item SkipGram models: \textit{LINE} \cite{tang2015line} is a static network embedding model  and \textit{DNE} \cite{du2018dynamic} is an extension of LINE to update node representations efficiently in dynamic environment.
    \item GNNs (Retrained): \textit{GraphSAGE} \cite{hamilton2017inductive}, \textit{GCN} \cite{kipf2016semi} and \textit{EvolveGCN} \cite{pareja2019evolvegcn} are retrained in the entire network at each time step. \textit{GraphSAGE} is also named as \textit{RetrainedGNN}, standing as an upper bound of our model. \textit{EvolveGCN} captures the dynamism of the graph sequence by using an RNN to evolve the GCN parameters.
    \item GNNs (Incremental): Simple incremental learning methods based on GraphSAGE in streaming networks. \textit{PretrainedGNN} represents a GraphSAGE learned at the first time step and not trained any longer. \textit{SingleGNN} is a GraphSAGE trained at every time step individually. \textit{OnlineGNN} is a variant of GraphSAGE trained in a online manner but without any consideration for pattern detection or knowledge consolidation. 
    \item \textit{ContinualGNN} is our model based on GraphSAGE trained incrementally via continual learning.
\end{itemize}

\subsubsection{Parameter Settings}
We set the embedding size of all models in all datasets to 64. 
For LINE and DNE, we use both first-order and second-order proximity representations and concatenate them.
Parameters of all deep models are set consistently. The number of hidden layers is 2, where the size of each layer is set to 64. For each layer, we use mean aggregator with 10 neighbors sampled. For EvolveGCN, we use the version EvolveGCN-O that achieves better performance. For our model, we set $\lambda=(80, 100, 100, 200)$ and $m=(250, 500, 500, 250)$ for Cora, Elliptic, DBLP and Synthetic respectively. We use the ratio of influenced nodes of all potential nodes, including changed nodes and their $L$-hop neighbors, to reflect the detection threshold $\delta$ and set $\%=(0.8,0.8,0.8,0.8)$.

\subsection{Experimental Results}

We employ four datasets for \textit{node classification} task to demonstrate the superior of our model. Data is randomly split into training set and testing set with the ratio $7:3$ at each time step. For supervised models, node labels are inferred directly by models. For unsupervised models, representations obtained from models are classified by Logistic Regression in sklearn package.

\subsubsection{Node Classification}

Table \ref{tab:nc} shows the averaged F1 and accuracy of all baselines. Comparing with SkipGram models and other incremental GNNs, our model, ContinualGNN achieves the best performance and gets the closest results to the upper bound, RetrainedGNN, proving that our model has great advantages in streaming networks. DNE is also an incremental learning method on streaming networks. However, it ignores node attributes when learning representations, and is not a deep model, so the result is uncompetitive compared to the GNN models. It should be noted that EvolveGCN performs poorly on Cora and Synthetic because synthetic streaming data lack temporal patterns.

Figure \ref{fig:nc} shows the accuracy of node classification of four incremental GNNs on continuous time steps. On Cora and Elliptic, it can be seen that the results of our model can maintain a good result over time, indicating that our model can learn new knowledge while consolidating existing patterns. While the results of PretrainedGNN and OnlineGNN both declined to a certain extent, because PretrainedGNN does not learn new knowledge during network evolution, and OnlineGNN fails to maintain old knowledge well, leading to catastrophic forgetting. On DBLP, the results of PretrainedGNN will be very poor, indicating that a large number of new patterns appear after the first time step. The performance of OnlineGNN generally looks good, however, it is very unstable compared to our model, indicating that our model also has a certain smoothing ability, and can maintain stable representations. In addition, on Synthetic, due to the dramatic changes in the network, the results of all models will be attenuated to a certain extent. The overall results of OnlineGNN are getting worse but the results of our model can recover after a short-term deterioration, indicating our model integrate new and historical knowledge well.

Besides, Table \ref{tab:nc} and Figure \ref{fig:nc} also show running time per epoch of each model.
Even though retrained GNNs like RetrainedGNN as the upper bound of our model can maintain great performance over time, their training complexity is excessively high and about 5 times to the incremental ones, which is impractical in large-scale networks. Our proposed model, ContinualGNN can be trained with high efficiency in a incremental manner so that it provides opportunities for various real-world applications. 

\subsubsection{Case study on Synthetic Data}
\label{sec:exp_synthetic}

We conduct experiments on the synthetic network. We construct a specific network, and then through special case study and visualization to prove that our model can deal with the problem of catastrophic forgetting in a streaming network where its distribution shifts over time.

We first introduce the generation method of the synthetic network. The ER graph generator \cite{gilbert1959random} is used to generate new networks for the 0-7th time steps, the degrees of the two categories are 4 and 10, and the community graph generator \cite{fortunato2010community} is used to generate new networks for the 8-23th time steps. The probability of edge generation within and between communities is 0.02 and 0.001 respectively. On the 0-15th time steps, the first dimension of the node attributes is generated by the normal distribution $\mathbb{N}(-1,1)$, and on the 16-23th time steps it is generated by the normal distribution $\mathbb{N}(1,1)$. The other dimensions all follow the standard normal distribution. Through this generation method, the networks of the 0-7th and 8-15th time steps have different network topology distributions, and the networks of the 8-15th and 16-23th time steps have different node attribute distributions.

Figure \ref{fig:acc_v0} shows the performance of nodes arriving at time $t_0$. As time evolves and the GNN model is updated, the classification accuracy changes. It can be seen that at the 0-7th time steps, because the distribution of the network is roughly unchanged, the accuracy rate gradually increases as the training data becomes more. But at time $t_8$ and $t_9$, the network patterns occur changes. For OnlineGNN, training the GNN on new nodes makes the model no longer applicable to the nodes at $t_0$, and catastrophic forgetting occurs, so its classification accuracy greatly decreases. On the contrary, our model, ContinualGNN, its classification accuracy can remain stable.

Figure \ref{fig:vis_v0} also shows the visualization of representations obtained at two adjacent time steps of nodes appearing at $t_0$. At time $t_7$, the GNN model can well separate the two types of nodes. At time $t_8$, when the distribution of the network structure changes, the representations of nodes belonging to two labels through OnlineGNN cannot be distinguished, while ContinualGNN can still distinguish the two kinds of nodes. It demonstrates that when the structure distribution changes, our model can well maintain historical knowledge and avoid catastrophic forgetting.

Similarly, Figure \ref{fig:acc_v8} shows the node classification results of the nodes at $t_8$ with the evolution of time. Figure \ref{fig:vis_v8} shows the visualization results of the nodes of the 8-th time step at two adjacent time steps. They all prove that when the distribution of node attributes changes, ContinualGNN can consolidate the existing knowledge and maintain satisfactory results.

\subsubsection{New Pattern Detection}

We analyzed the accuracy of node classification and the running time of the detection algorithm under different detection thresholds and different detection methods during the detection of new patterns on Synthetic and Cora. \textit{Naive} means to directly calculate the scores of all nodes in the network according to the Equation \ref{equ:influenced_node_set}, \textit{BFS} refers to calculate the scores of all neighbors of changed nodes, and \textit{Approximation} represents our approximate calculation method. 

The results are shown in Figure \ref{fig:detect}. On the one hand, it can be seen that as a larger proportion of new patterns are added during training, the accuracy will also be higher.
On the other hand, it can also be seen that our \textit{Approximation} is comparable to \textit{Naive} and \textit{BFS} in accuracy, indicating that our algorithm can approximate the score of the degree of node influence. But at the same time, thanks to the approximation algorithm only calculating a small number of nodes in the network, the time complexity is lower, so the effectiveness of the approximation algorithm is proved. But notice that on Synthetic, the detection time of \textit{Approximation} and \textit{BFS} is close and this is because the number of changed nodes and influenced nodes are similar in the network.

\begin{figure}[t]
\centering
\subfigure[Accuracy on Synthetic]{
    \label{fig:detect_synthetic_acc}
    \includegraphics[width=0.48\columnwidth]{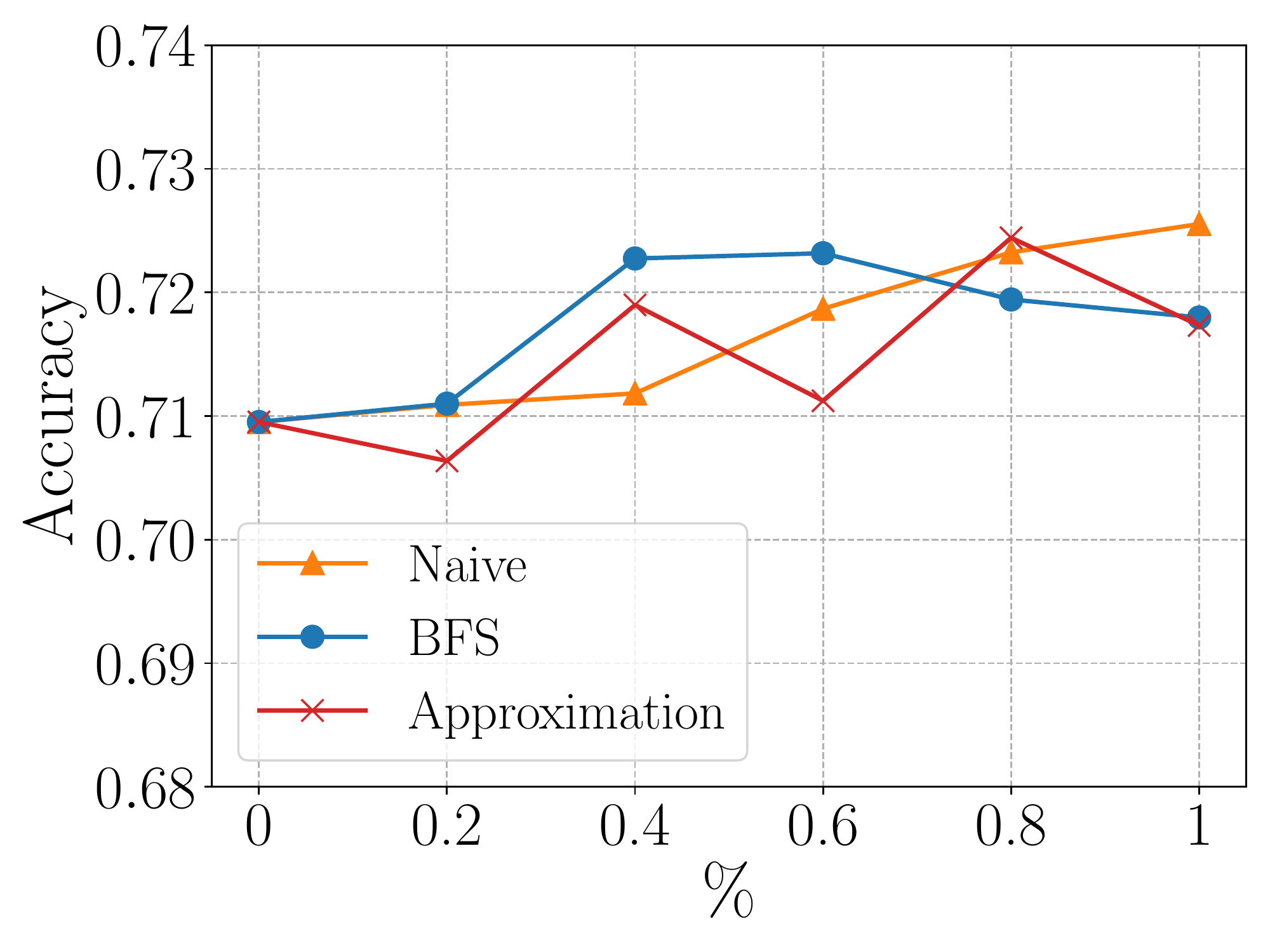}}
\subfigure[Detection Time on Synthetic]{
    \label{fig:detect_synthetic_time}
    \includegraphics[width=0.48\columnwidth]{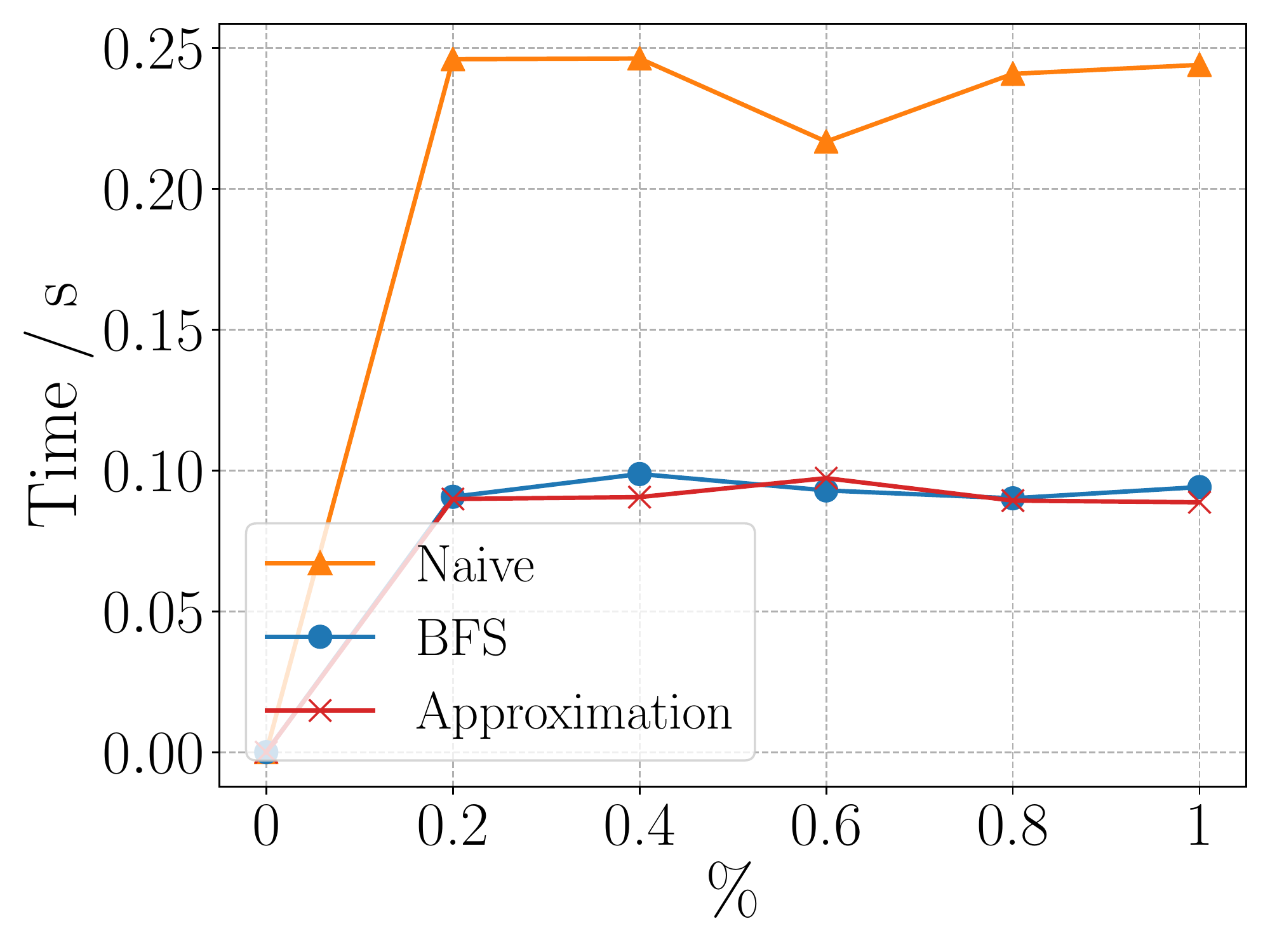}}

\subfigure[Accuracy on Cora]{
    \label{fig:detect_cora_acc}
    \includegraphics[width=0.48\columnwidth]{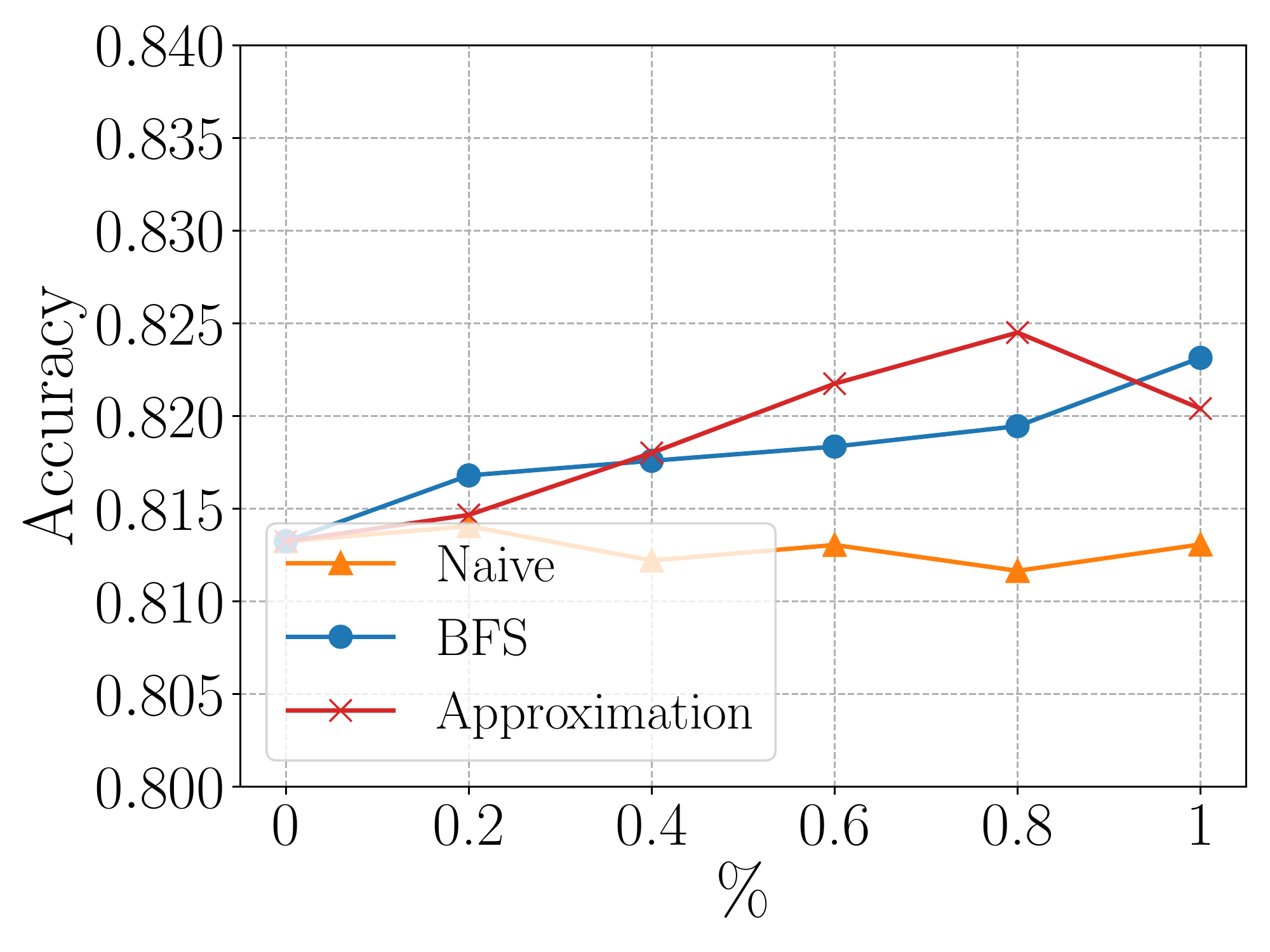}}
\subfigure[Detection Time on Cora]{
\label{fig:detect_cora_time}
    \includegraphics[width=0.48\columnwidth]{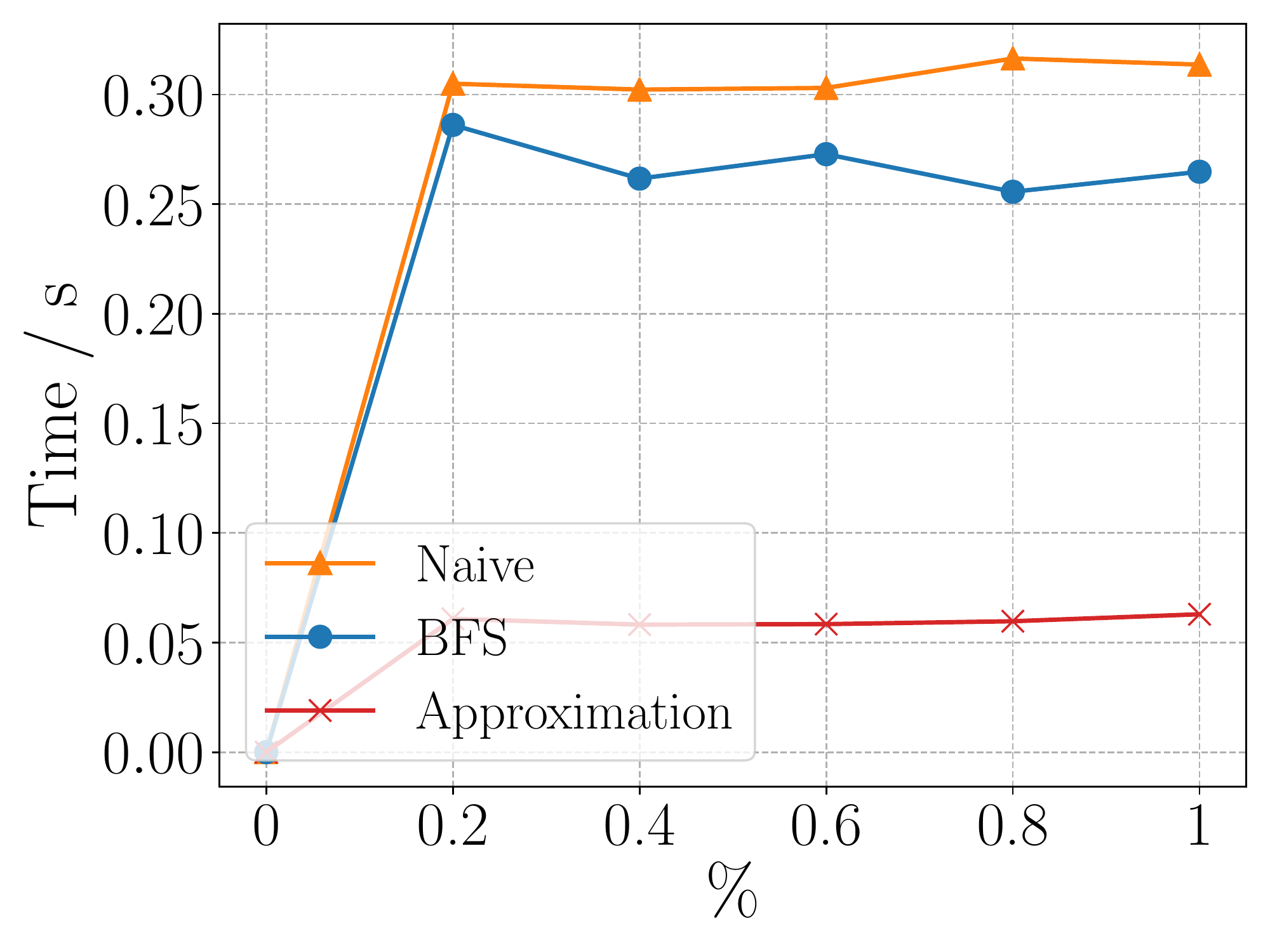}}
    \caption{Analysis of New Pattern Detection}
\label{fig:detect}
\end{figure}

\subsubsection{Existing Pattern Consolidation}

Next, on synthetic data, we discuss the impact of different views during the consolidation of the existing patterns in steaming networks.

We first analyze from a data perspective. Figures \ref{fig:mem_acc} and Figure \ref{fig:mem_time} respectively show the accuracy and running time of node classification on synthetic data, under different sampling strategies and different memory capacities $m$. On the one hand, as the memory capacity increases, the more representative nodes are saved, so the accuracy of the model is higher and the running time is longer. In order to balance the effectiveness and efficiency, we choose the memory size of 250. On the other hand, \textit{Random sampling} refers to simple random sampling, \textit{Hierarchical sampling} means hierarchical sampling without importance, and \textit{Step-wise sampling} is our step-wise sampling strategy. It shows that the sampling strategy we designed can achieve better results, and there is not much difference in the running time of the three. This proves that the step-wise sampling strategy can select nodes that can consolidate historical knowledge more efficiently and stably in streaming networks.

We also analyze the effects of different kinds of regularization terms and different regularization weights. The results are shown in Figures \ref{fig:lambda_acc}. The weighted regularization method we proposed has a better effect than the L2 regularization or no constraint, that is, $\lambda=0$. It proves that our method can better constrain the model and make it have a stronger generalization ability for knowledge preservation. At the same time, the best $\lambda$ value on Synthetic is 200.

Besides, we show the effect of saving existing knowledge from different views on node classification results in Figure \ref{fig:multi_acc}. It can be seen that from either data-view or model-view, the model results are improved to a certain extent. But the combination of the two can make up for each other's deficiencies and achieve the best results.

\begin{figure}[t]
\centering
\subfigure[Accuracy from data-view]{
    \label{fig:mem_acc}
    \includegraphics[width=0.48\columnwidth]{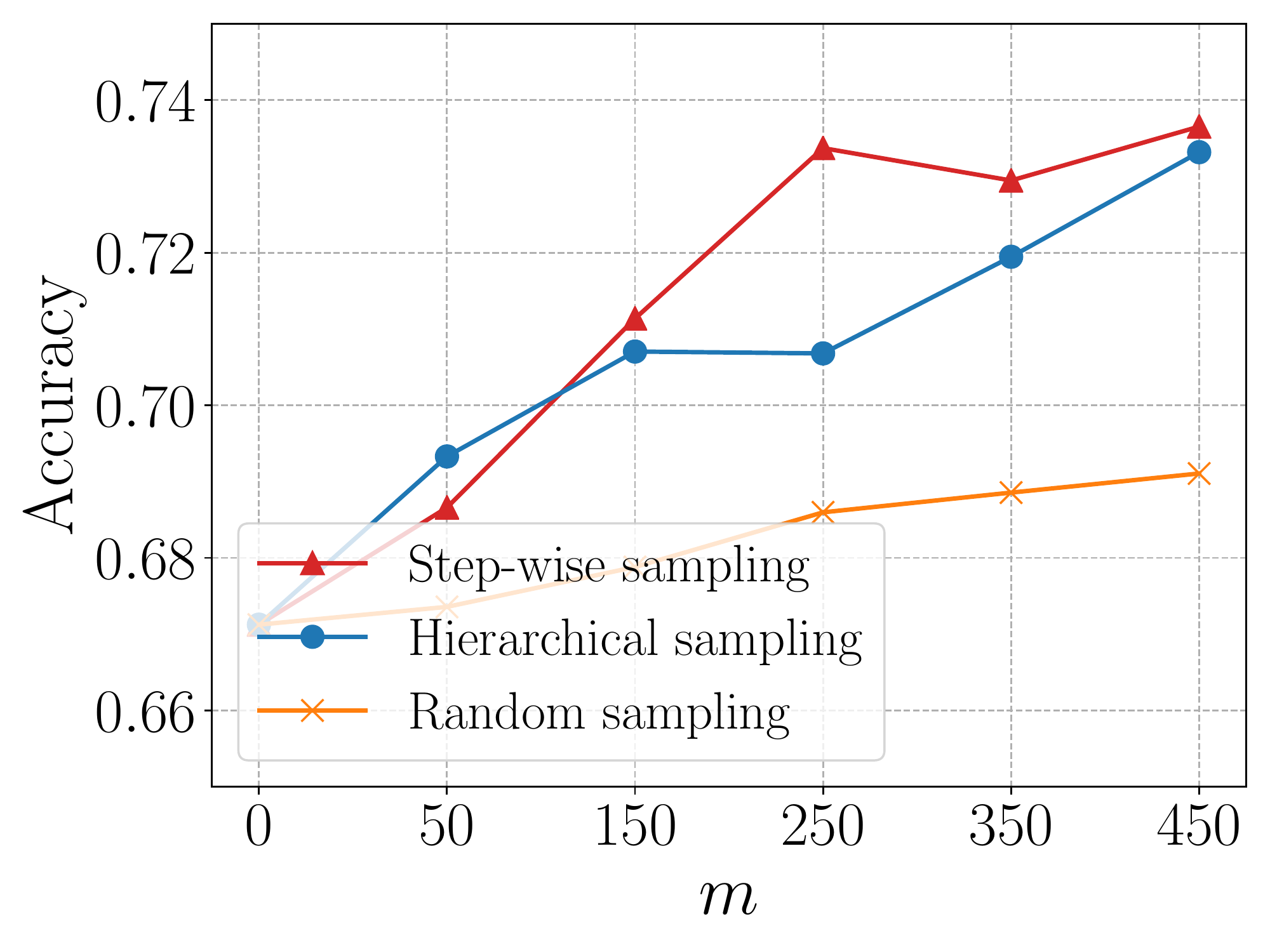}}
\subfigure[Running time from data-view]{
    \label{fig:mem_time}
    \includegraphics[width=0.48\columnwidth]{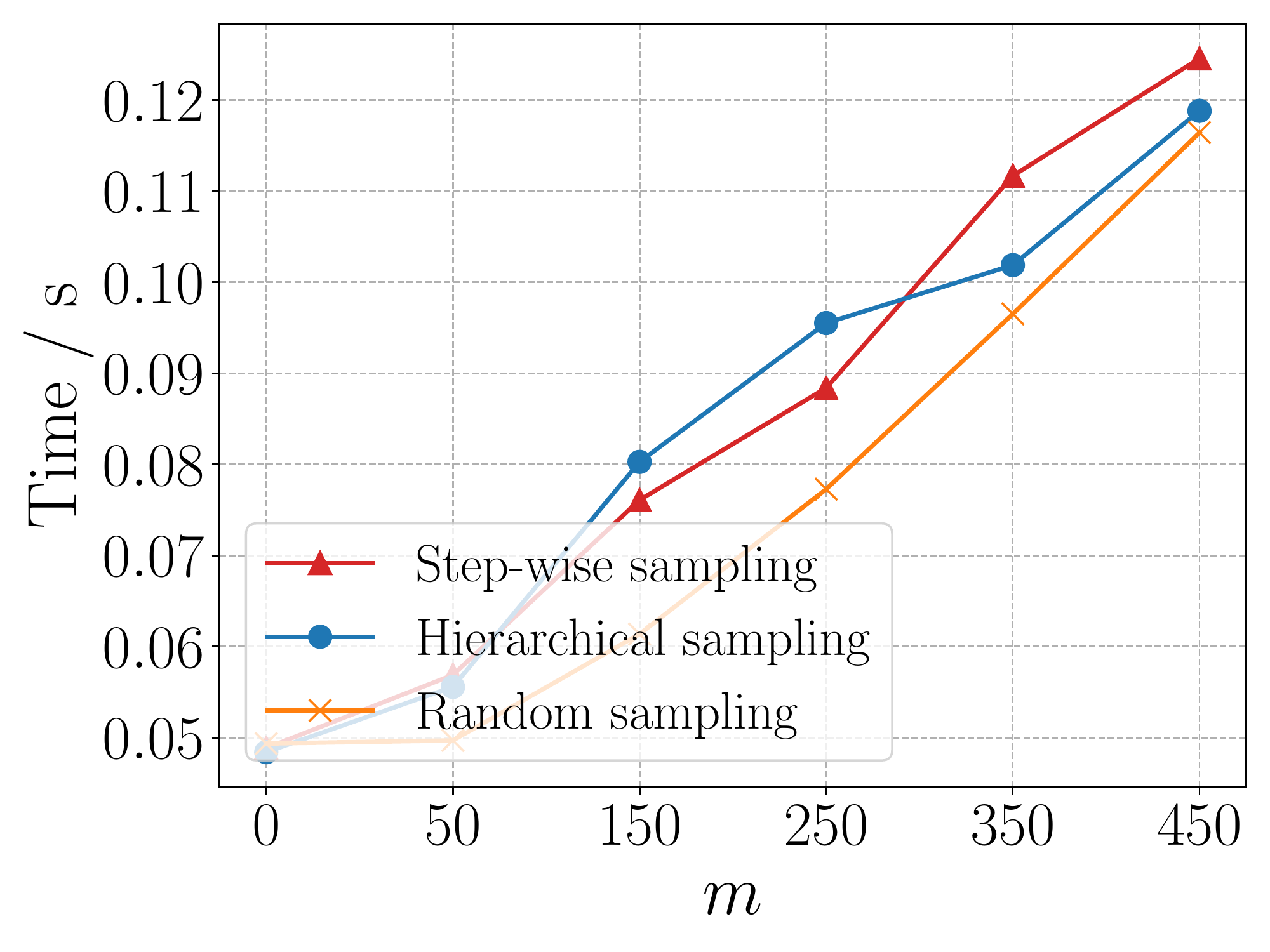}}

\subfigure[Accuracy from model-view]{
    \label{fig:lambda_acc}
    \includegraphics[width=0.48\columnwidth]{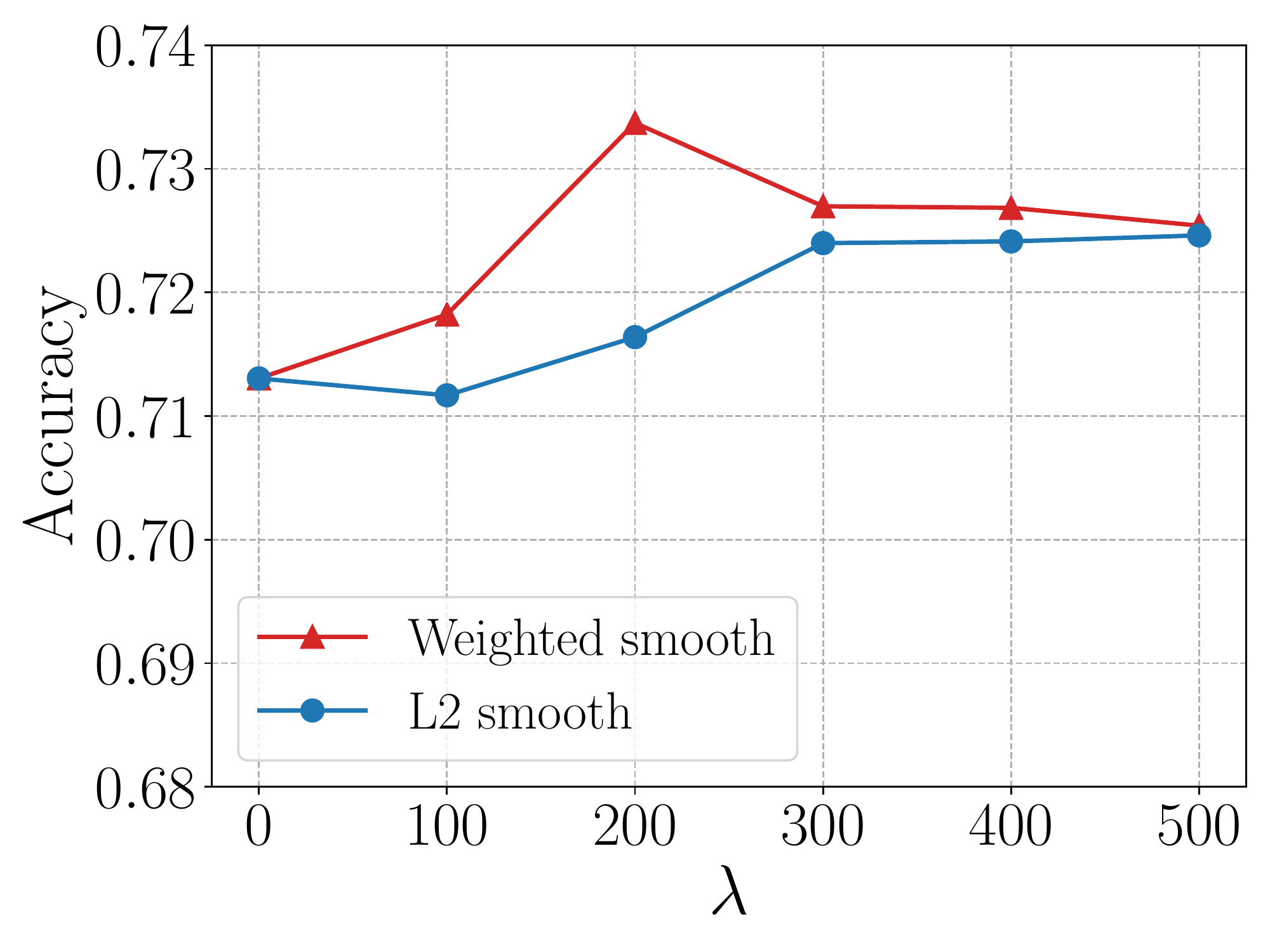}}
\subfigure[Accuracy from multi-view]{
    \label{fig:multi_acc}
    \includegraphics[width=0.48\columnwidth]{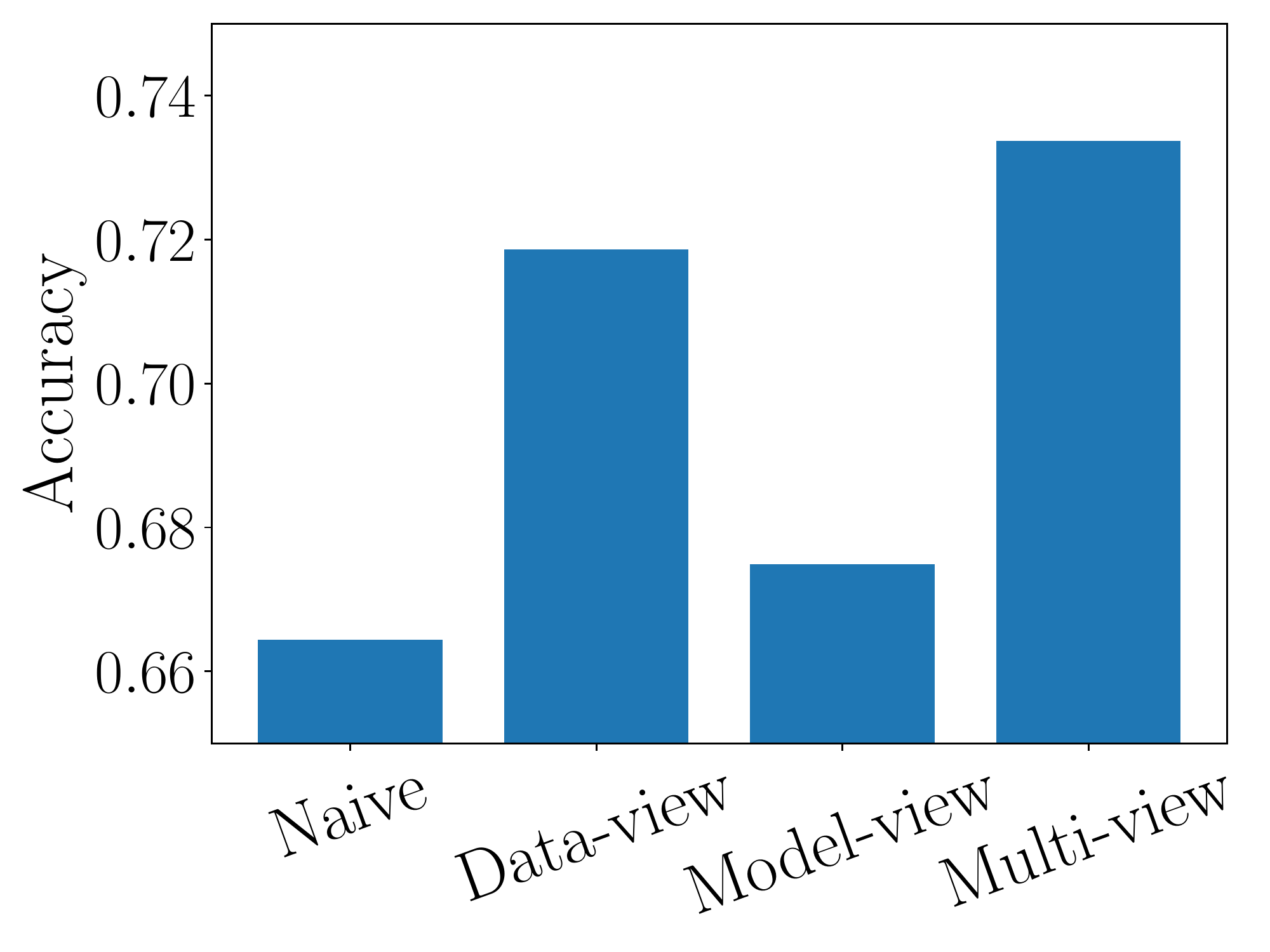}}    
\caption{Analysis of Existing Pattern Consolidation}
\end{figure}

\subsubsection{Scalability}    


Finally, we discuss the scalability of our model from two aspects. 
As shown in Figure \ref{fig:scalability_1}, we evaluate the running time in synthetic networks with different numbers of nodes, which are collected over all training data. It shows that the running time of retraining models increases much faster than our model when the network scale becomes large. But our model can keep converging in a short time, which indicates that the scalability of our proposed model is empirically good. 
Then, in Figure \ref{fig:scalability_2}, we observe the running time with different sizes of streaming data per time step. Obviously, the running time of our model is relative to the size of data per time step and our model have good scalability on real-world streaming networks where the streaming data per time step is often much fewer than the total scale of networks.

\begin{figure}[!htbp]
\centering
\subfigure[Network size]{
    \label{fig:scalability_1}
    \includegraphics[width=0.48\columnwidth]{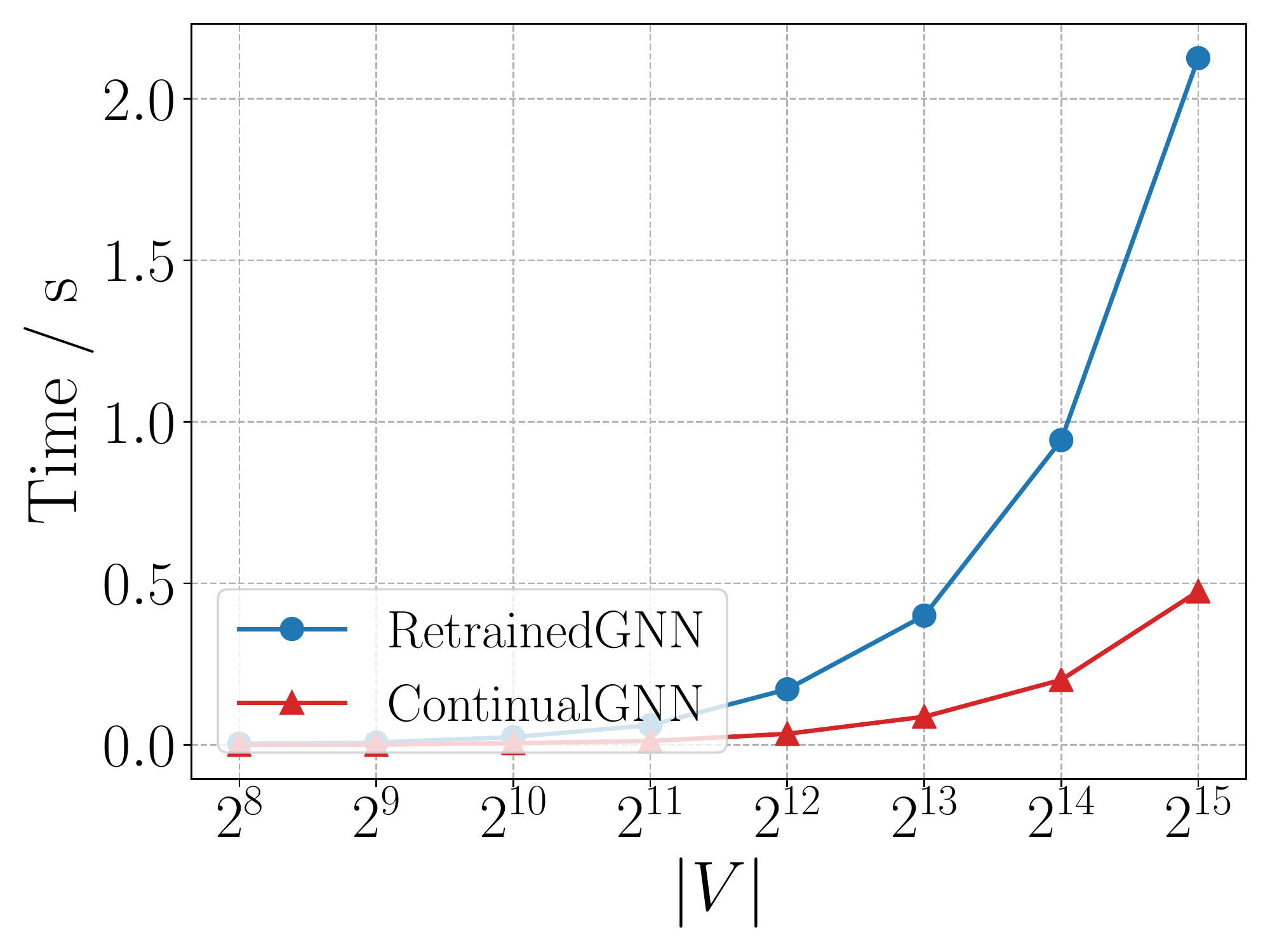}}
\subfigure[Stream size]{
    \label{fig:scalability_2}
    \includegraphics[width=0.48\columnwidth]{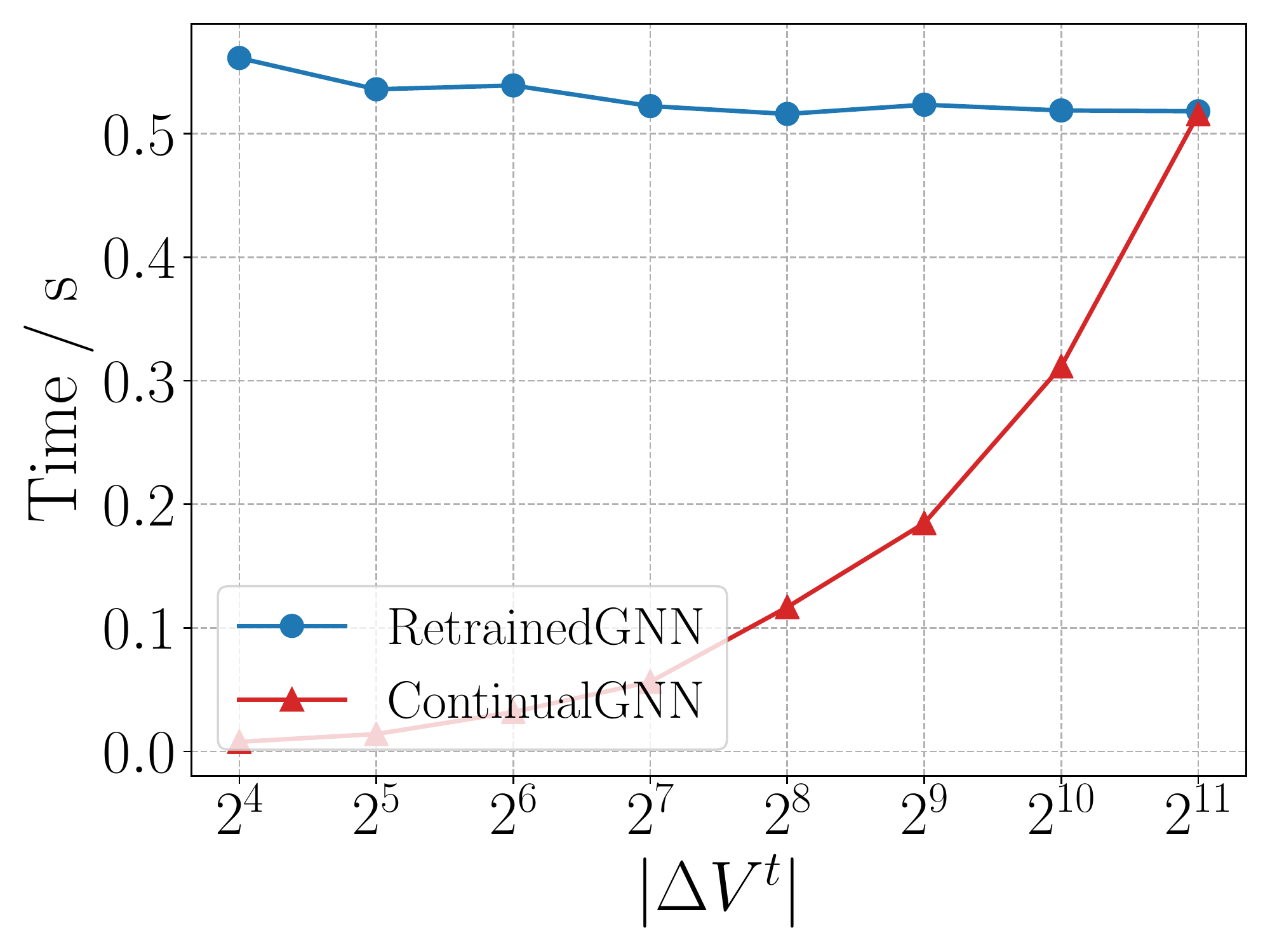}}
\caption{Scalability}
\end{figure}

\section{Conclusion}
In this paper, we address the problem of incremental learning for GNN models based on continual learning when data distribution shifts and new patterns appears over time. We designed an approximation algorithm based on traversal, which can quickly detect new patterns that may exist in a streaming network. Then, based on two complementary perspectives, we propose a method to consolidate the existing knowledge in the network, that is, constraining the current model based on a small portion of historical data and the previous model. 
Experiments on both real-world and synthetic networks prove that our model is more efficient than retraining models but achieves comparable results.



\begin{acks}
We are thankful to Yi Li for his helpful suggestions. 
This work was supported by the National Natural Science Foundation of China (Grant No. 61876006).
\end{acks}

\bibliographystyle{ACM-Reference-Format}
\bibliography{acmart}

\end{document}